\documentclass[11pt]{article}

\usepackage[letterpaper,margin=1in]{geometry}

\usepackage{amsmath,amssymb,amsfonts,amsthm}
\usepackage{algorithm}
\usepackage[noend]{algpseudocode}
\usepackage{booktabs}
\usepackage{graphicx}
\usepackage{enumitem}
\usepackage[round,sort&compress,authoryear]{natbib}
\usepackage{xcolor}
\usepackage{url}
\usepackage[colorlinks=true,linkcolor=blue!50!black,
            citecolor=blue!50!black,urlcolor=blue!50!black]{hyperref}

\theoremstyle{plain}
\newtheorem{theorem}{Theorem}
\newtheorem{lemma}{Lemma}
\newtheorem{corollary}{Corollary}

\usepackage[compact]{titlesec}
\titlespacing{\section}{0pt}{10pt}{5pt}
\titlespacing{\subsection}{0pt}{8pt}{4pt}
\titlespacing{\subsubsection}{0pt}{6pt}{3pt}

\algrenewcommand\algorithmicrequire{\textbf{Input:}}
\algrenewcommand\algorithmicensure{\textbf{Output:}}

\newcommand{\ccsign}{{\ooalign{\hfil\raise.07ex\hbox{c}\hfil\crcr\mathhexbox20D}}}
\newcommand{\codeurl}{\url{https://github.com/hferdous/DCD}}

\title{DCD: Decomposition-based Causal Discovery from\\
       Autocorrelated and Non-Stationary Temporal Data}

\usepackage{authblk}

\author[1,2]{Muhammad Hasan Ferdous}
\author[2]{Ian Stockwell}
\author[1,2]{Md Osman Gani}
\affil[1]{Causal AI Lab, University of Maryland, Baltimore County, Baltimore, MD, USA}
\affil[2]{Department of Information Systems, University of Maryland, Baltimore County, Baltimore, MD, USA}
\affil[ ]{\texttt{h.ferdous@umbc.edu, istock1@umbc.edu, mogani@umbc.edu}}

\date{}

\begin{document}

\maketitle

{\renewcommand{\thefootnote}{}%
\footnotetext{This is the extended version of a paper accepted at the SSTDM-2026 workshop, IEEE International Conference on Data Mining Workshops (ICDMW) 2026. The appendices add full results with standard deviations and the extended ablation studies.}%
\footnotetext{Research reported in this publication was supported by the National Institute on Minority Health and Health Disparities of the National Institutes of Health under Award Number R01MD019814. The content is solely the responsibility of the authors and does not necessarily represent the official views of the National Institutes of Health.}%
\footnotetext{\ccsign\,2026 IEEE. Personal use of this material is permitted. Permission from IEEE must be obtained for all other uses, in any current or future media, including reprinting/republishing this material for advertising or promotional purposes, creating new collective works, for resale or redistribution to servers or lists, or reuse of any copyrighted component of this work in other works.}}

\begin{abstract}
Multivariate time series in finance, climate science, and healthcare often combine long-term trends, seasonal cycles, and short-term fluctuations, which makes causal inference difficult when the data drift and are correlated over time. Most causal discovery methods operate on the raw series and are prone to spurious edges and misattributed dependencies. We introduce a decomposition-based causal discovery (DCD) framework that separates each series into trend, seasonal, and residual components and analyzes each with a method suited to it: stationarity tests for trends, kernel-based dependence tests for seasonality, and constraint-based discovery for residuals. The component-level results are merged into a multi-scale causal graph that distinguishes variable-to-variable effects from links that only reflect shared temporal drivers. We show that the framework recovers the true causal graph under stated conditions, with an error that stays small when the components separate cleanly and grows gradually as that separation weakens. Across synthetic benchmarks and real-world climate and electricity datasets, DCD recovers causal structure more accurately than competing methods under strong trends and serial correlation. Code is available at \codeurl.
\end{abstract}


\section{Introduction}

Multivariate time series in economics, environmental science, finance, and healthcare usually mix long-term trends, seasonal cycles, and short-term fluctuations, and the causal structure at one of these scales can differ from the structure at another. For forecasting and decision-making it is the causal relationships between variables that matter, not merely their correlations. Most causal discovery (CD) methods do not separate long-term effects, seasonal dependencies, and transient interactions, so on such data they often return spurious or misattributed relationships~\citep{granger1969investigating, spirtes_causation, runge2020discovering}. Deep sequence models~\citep{vaswani2017attention, wu2021autoformer, liu2024generative, piao2024fredformer, wen2022robust} capture long-range dependencies, but their objectives remain correlational, and separating spurious correlation from causal influence remains difficult~\citep{dhaou2021causal}.

We propose a \textbf{decomposition-based causal discovery (DCD) framework}. DCD decomposes each series into trend, seasonal, and residual components and analyzes each with a method suited to it. It applies \emph{stationarity tests}~\citep{kpss} to the trend to detect long-term temporal dependence, \emph{kernel dependence measures}~\citep{scholkopf2002learning} to the seasonal component to detect cyclic forcing, and \emph{constraint-based discovery}~\citep{runge2020discovering} to the residual to recover short-term causal effects. The component graphs are then merged into one multi-scale causal graph. Our contributions are as follows.
\begin{itemize}[leftmargin=*,itemsep=2pt,topsep=2pt]
    \item \textbf{Decomposition-based causal discovery.} A framework that separates trend, seasonal, and residual components and performs targeted inference at each scale.
    \item \textbf{Identifiability analysis.} A treatment of decomposition as spectral filtering, with identifiability guarantees under linear Gaussian dynamics and a Structural Hamming Distance (SHD) bound in terms of the leakage parameter $\varepsilon$ and the Type-I and Type-II error rates of the CI tests.
    \item \textbf{Empirical validation.} On synthetic benchmarks and two real-world datasets, DCD recovers causal structure better than PCMCI+~\citep{runge2020discovering}, CD-NOD~\citep{huang2020causal}, DYNOTEARS~\citep{pamfil2020dynotears}, and other constraint- and score-based baselines. Isolation and sensitivity studies show that the gains come from multi-scale integration, not from decomposition preprocessing alone.
\end{itemize}

\begin{figure}[t]
    \centering
    \includegraphics[width=0.85\textwidth]{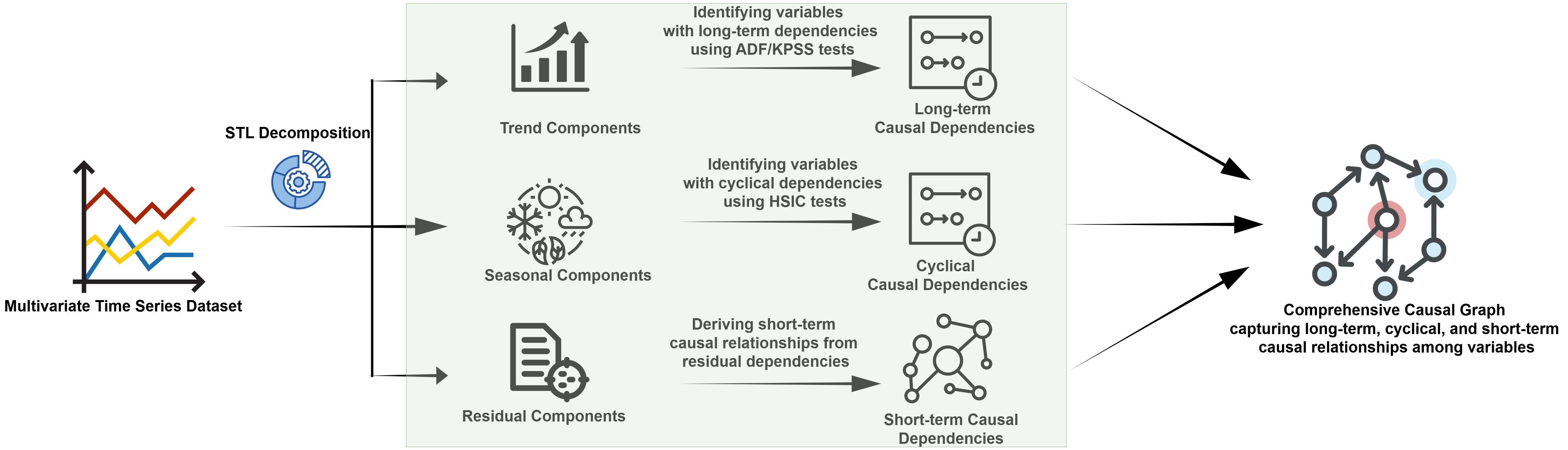}
    \caption{Overview of DCD. STL splits each series into trend, seasonal, and residual components. ADF/KPSS and HSIC tests mark time-proxy edges ($t\to X$), PCMCI+ on the residuals recovers mechanistic edges ($X\to Y$), and the component graphs are merged into one multi-scale graph.}
    \label{fig:DCD_framework}
\end{figure}

\section{Related Work}

\textbf{Decomposition and representation learning.}
Classical decomposition separates a series into trend, seasonal, and irregular parts~\citep{stl}. STL and moving-average methods~\citep{stl, hyndman2021forecasting} assume smooth behavior, and Empirical and Variational Mode Decomposition~\citep{emd, vmd} extract oscillatory modes for frequency analysis rather than causal structure. Transformers~\citep{vaswani2017attention}, Autoformer~\citep{wu2021autoformer}, Temporal Fusion Transformers~\citep{lim2021temporal}, and N-BEATS~\citep{oreshkin2019nbeats} learn multi-scale representations, but attention weights and learned basis blocks do not identify causal structure. Pipelines that decompose a series do so for forecasting, while discovery methods run on raw data where trend and seasonality confound the CI tests.

\textbf{Causal discovery in time series.}
Autocorrelation, temporal dependence, and non-stationarity complicate discovery from observational series~\citep{runge2019inferring}. Granger causality~\citep{granger1969investigating, zhang2024learning} relies on predictive improvement and assumes linearity and stationarity. Constraint-based methods such as PC, FCI~\citep{peters_causal}, and RFCI~\citep{colombo2012learning} and score-based methods such as GES and FGES~\citep{GES, fges} produce spurious edges under strong autocorrelation. PCMCI+~\citep{runge2020discovering}, DYNOTEARS~\citep{pamfil2020dynotears}, and CDANs~\citep{Cdans} separate lagged from contemporaneous effects, but none decomposes the series into temporal scales, so trends and seasonality can mask or mimic relationships.

\textbf{Non-stationary discovery.}
CD-NOD~\citep{huang2020causal} adds a surrogate time or domain index to absorb distribution shift, CD-NOTS~\citep{sadeghi2025cdnots} extends it to lagged and instantaneous links in time series, and CDANs~\citep{Cdans} identifies lagged and contemporaneous links together with mechanisms that change over time. These methods treat non-stationarity as a nuisance to condition away, which can collapse genuine temporal drivers onto the surrogate index. DCD keeps the driver explicit ($t\to X$) and recovers the lagged graph on a cleaned residual. The two views are complementary, since such methods are needed when the residual mechanism itself shifts.

\textbf{Cointegration and spectral causality.}
Regressions between independent trending series can look significant yet be spurious~\citep{granger1974spurious}, which motivated cointegration and error-correction models that separate long-run equilibria from short-run dynamics~\citep{engle1987cointegration, johansen1991estimation}. Treating the trend as an exogenous driver ($t\to X$) serves a similar purpose in DCD, but non-parametrically and without a cointegrating rank. Frequency-domain measures split dependence and feedback by spectral band~\citep{geweke1982measurement}, and our separability assumption (A2) is the structural counterpart of that idea. The Spectral Independence Criterion~\citep{shajarisales2015telling} uses spectral content for orientation, and functional-model methods such as TiMINo~\citep{peters2013causal} offer an alternative residual search. DCD uses the bands only to remove confounding before a standard time-domain search, and uses PCMCI+ for its handling of autocorrelation.

\section{Preliminaries}
\label{sec:prelim}

\textbf{Time-series SCM.} A structural causal model (SCM) over $\mathbf{X}(t)=\{X_1(t),\dots,X_d(t)\}$ writes each variable as a function of its causal parents and an independent noise, $X_i(t)=f_i(\mathrm{PA}(X_i(t)),\xi_i(t))$, where parents may be contemporaneous ($X_j(t)$) or lagged ($X_j(t-\tau)$, $\tau\ge1$). Unrolling over a window $\tau\in\{0,\dots,\tau_{\max}\}$ gives the time-series (window) graph, whose nodes are lagged variables and whose edges encode direct causal influence~\citep{runge2020discovering,peters_causal}. Under causal stationarity the mechanisms and noise distributions do not change over time, so one repeated graph describes the dynamics~\citep{runge2019inferring}. Non-stationarity in the observed series then enters through the trend and seasonal components, which DCD isolates first.

\textbf{Conditional independence and faithfulness.} Constraint-based discovery infers structure from conditional independence (CI) relations $X\perp\!\!\!\perp Y\mid \mathbf{Z}$ under the causal Markov and faithfulness conditions, which link $d$-separation in the graph to CI in the distribution~\citep{spirtes_causation,peters_causal}. Autocorrelation reduces the effective sample size and biases CI tests toward false positives, so it must be controlled, e.g.\ by the momentary CI test of PCMCI~\citep{runge2019detecting}, which PCMCI+~\citep{runge2020discovering} extends.

\textbf{Evaluation.} We compare an estimated graph $\widehat{G}$ with the ground truth $G^\star$ by the Structural Hamming Distance (SHD), the number of edge insertions, deletions, and reversals between them~\citep{tsamardinos2006max}, and by the true positive rate $\mathrm{TPR}=\mathrm{TP}/(\mathrm{TP}+\mathrm{FN})$ and false discovery rate $\mathrm{FDR}=\mathrm{FP}/(\mathrm{TP}+\mathrm{FP})$. Higher TPR and lower FDR and SHD are better.

\section{Problem Formulation and Theory}
\label{sec:theory}

\subsection{Formal Problem Setup}

Let $\mathbf{X}(t)$ be indexed by $t\in\mathbb{Z}$, with $t$ an exogenous index that carries the shared temporal forcing (calendar time, climate cycles, and other common drivers). We take the data-generating process to be the additive SCM
\begin{equation}
\begin{aligned}
X_i(t) &= T_i(t) + S_i(t) + R_i(t),\\
T_i(t) &= g_i(t) + \omega^{T}_i(t),\qquad
S_i(t) = h_i(t) + \omega^{S}_i(t),\\
R_i(t) &= f_i\!\big(\mathrm{PA}^{R}_i(t),\,\xi_i(t)\big),
\end{aligned}
\label{eq:scm}
\end{equation}
for $i\in\{1,\dots,d\}$, where $g_i$ is a low-frequency trend map of time, $h_i$ a periodic seasonal map of period $P_i$, and the residual mechanism $f_i$ acts on residual parents $\mathrm{PA}^{R}_i(t)\subseteq\{R_j(t-\tau): \tau\ge0,\,(j,\tau)\neq(i,0)\}$ with independent noises $\xi_i$. The band-limited noises $\omega^{T}_i,\omega^{S}_i$ stay in the trend and seasonal ranges.

Eq.~\eqref{eq:scm} fixes three edge classes in the true graph $G^\star$, with $E^\star = E^\star_T \cup E^\star_S \cup E^\star_R$. The trend and seasonal layers depend only on $t$, so they contribute \emph{time-proxy edges} $t\to X_i$, in $E^\star_T$ when $g_i$ is non-constant and in $E^\star_S$ when $h_i\neq0$. The residual layer carries the \emph{mechanistic edges} $X_j\to X_i$ of $E^\star_R$, given by $\mathrm{PA}^{R}_i$. Routing all shared temporal forcing through the single parent $t$ stops two variables that merely move together in time from being joined directly. ``Multi-scale'' thus means separating temporal forcing from intrinsic mechanism, not discovering cross-variable seasonal causality. The algorithm of Section~\ref{sec:dcd_framework} targets the general form~\eqref{eq:scm}. The identifiability analysis below uses its linear Gaussian special case (A3), in which each $f_i$ is linear and each $\xi_i$ is Gaussian.

\subsection{Assumptions}

\noindent\textbf{A1. Causal invariance and asymptotic independence.}
The mechanism governing $\mathbf{R}(t)$ is time-invariant, and $\mathbf{R}(t)$ is $\beta$-mixing, $\beta(k)\to0$ as $k\to\infty$, which CI estimators need for consistency~\citep{rio2017asymptotic,merlevede2011bernstein}. With $\mathcal{F}_{-\infty}^t$ and $\mathcal{F}_{t+k}^{\infty}$ the $\sigma$-algebras generated by $\{\mathbf{R}(s):s\le t\}$ and $\{\mathbf{R}(s):s\ge t+k\}$,
\begin{equation}
    \beta(k) = \sup_{t} \mathbb{E}\!\left[ \sup_{A \in \mathcal{F}_{t+k}^{\infty}} \left| P(A \mid \mathcal{F}_{-\infty}^t) - P(A) \right| \right].
    \label{eq:mixing}
\end{equation}

\noindent\textbf{A2. Spectral separability.}
The components have disjoint support in the frequency domain. With $f_Y(\omega)$ the spectral density of process $Y$ and cutoffs $\omega_1<\omega_2$,
\begin{equation}
    \mathrm{supp}(f_T) \subset [0,\omega_1), \quad
    \mathrm{supp}(f_S) \subset \textstyle\bigcup_{k\ge1}\left[\tfrac{2\pi k}{P}-\delta,\tfrac{2\pi k}{P}+\delta\right], \quad
    \mathrm{supp}(f_R) \subset (\omega_1,\pi]\setminus \mathrm{supp}(f_S).
    \label{eq:sep}
\end{equation}

\noindent\textbf{A3. Linear Gaussian dynamics and bounded leakage.}
For the identifiability bound we assume a linear Gaussian structural equation model (SEM), under which zero covariance implies independence. The leakage parameter $\varepsilon$ bounds the largest correlation $\rho$ between an estimated trend or seasonal component and an estimated residual,
\begin{equation}
    \lambda = \max_{i,j}\max\bigl\{|\rho(\widehat{T}_i,\widehat{R}_j)|,\,|\rho(\widehat{S}_i,\widehat{R}_j)|\bigr\} \le \varepsilon\ll 1 .
    \label{eq:leak}
\end{equation}

\noindent\textbf{A4. Causal modularity.}
Cross-scale causal influences are negligible or blocked by decomposition, i.e.\ $R_i(t-\tau)\not\rightarrow T_j(t)$ and $R_i(t-\tau)\not\rightarrow S_j(t)$ for any lag $\tau$.

\noindent\textbf{Remarks.} A1--A4 are sufficient, not necessary. A2 and A4 fail under cross-frequency coupling or residual$\to$trend feedback, and the sensitivity study (Section~\ref{sec:ablations}, Appendix~\ref{app:spectral_sensitivity}) shows that accuracy then declines gradually. A2 is also what separates $T_i$ from $S_i$, since both are functions of $t$. If their bands overlap, a temporal effect may be credited to the wrong time-proxy class. This changes only the label of a $t\to X_i$ edge, and the effect on $G_R$ is still measured by $\varepsilon$, which involves only correlations with the residual. A3 is needed only to equate zero covariance with independence. The pipeline itself uses LOESS smoothing and non-parametric CI tests such as CMI-knn and the Hilbert-Schmidt Independence Criterion (HSIC), and our experiments include non-linear settings.

\subsection{Identifiability Analysis}

\begin{lemma}[Projection of Seasonal Influence]
\label{lem:projection}
Consider a causal pathway $S_X(t)\to Y(t)$ where $S_X$ is a seasonal driver. Under A2 and A3, the decomposition of $Y(t)$ projects this influence onto $S_Y(t)$, orthogonal to the residual subspace $R_Y(t)$ up to order $\varepsilon$,
\[
R_Y(t)\perp\!\!\!\perp S_X(t)\mid S_Y(t) + \mathcal{O}(\varepsilon).
\]
\end{lemma}

\begin{proof}
Under A3 the pathway is linear, $Y(t)=(a\ast S_X)(t)+\eta(t)$, with $a$ a linear filter and $\eta$ gathering independent influences. A linear time-invariant filter multiplies the input spectrum by its gain and creates no new frequencies, so $a\ast S_X$ keeps the spectral support of $S_X$, which lies in the seasonal band by A2. STL absorbs it into $\widehat{S}_Y$, giving $\widehat{S}_Y\approx a\ast S_X$ and $\widehat{R}_Y\approx\eta$. Finite-sample leakage moves some seasonal energy into the residual band, and by A3 $|\mathrm{Cov}(\widehat{R}_Y,\widehat{S}_X)|\le \varepsilon\,\sigma_{\widehat{R}_Y}\sigma_{\widehat{S}_X}$, so the conditional mutual information $I(R_Y;S_X\mid S_Y)$ vanishes as $\varepsilon\to 0$.
\end{proof}

For a nonlinear $f$ the conclusion still holds when $S_X$ is exactly periodic, since $f(S_X)$ then has period $P$, its harmonics lie in the seasonal support of A2, and its mean only shifts the level of $\widehat{T}_Y$. With band-limited seasonal noise, a nonlinear $f$ also creates low-frequency terms, which the lemma does not cover.

\textbf{Empirical check of Lemma~\ref{lem:projection}.} We simulate a bivariate system in which a seasonal component drives another variable and check that STL assigns the effect to $S_Y$ rather than $R_Y$. The structural equations are
\begin{equation}
\begin{aligned}
S_X(t) &= A_X \sin(2\pi t/P), & T_X(t) &= \beta_X t/n, & R_X(t) &\sim \mathcal{N}(0,\sigma_X^2),\\
S_Y(t) &= 1.5\,S_X(t) + \eta_S(t), & T_Y(t) &= \beta_Y t/n, & R_Y(t) &= \gamma R_X(t-1) + \xi_Y(t),
\end{aligned}
\end{equation}
with $X=T_X+S_X+R_X$, $Y=T_Y+S_Y+R_Y$, $P=12$, and $n=1000$. The system has a seasonal path $S_X\to S_Y$ and a separate residual path $R_X(t-1)\to R_Y(t)$. STL assigns the seasonal effect to $\widehat{S}_Y$, the correlation between $\widehat{S}_X$ and $\widehat{R}_Y$ is $|\rho|\approx0.0008$ (Figure~\ref{fig:sim_verification}), and the residual stage recovers only $R_X\to R_Y$.

\begin{figure}[t]
    \centering
    \includegraphics[width=0.95\textwidth]{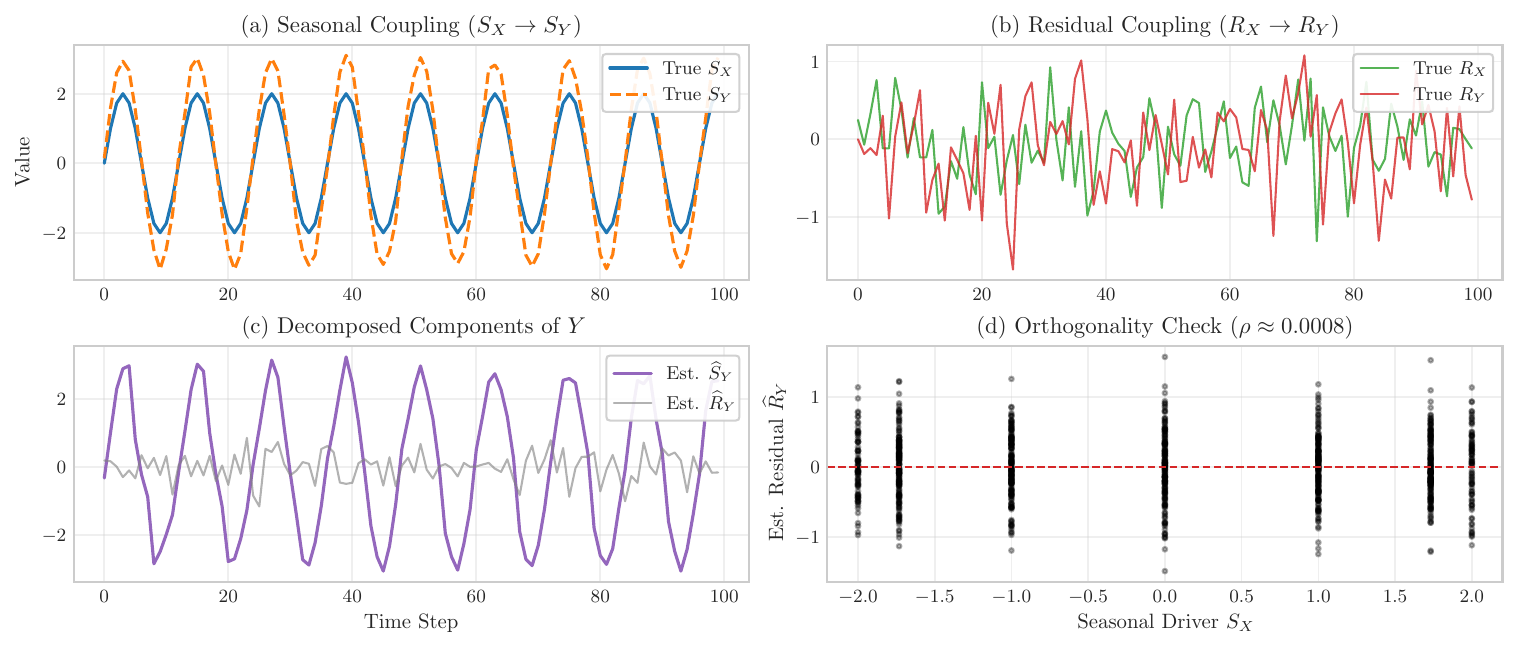}
    \caption{Empirical check of Lemma~\ref{lem:projection}. (a) The true seasonal component $S_X$ drives $S_Y$. (b) The true residual $R_X$ drives $R_Y$. (c) STL separates the estimated seasonal $\widehat{S}_Y$ and residual $\widehat{R}_Y$. (d) $S_X$ is uncorrelated with the estimated residual $\widehat{R}_Y$ ($\rho \approx 0.0008$).}
    \label{fig:sim_verification}
\end{figure}

\begin{lemma}[Weak Dependence of Residuals]
\label{lem:weak_dependence}
Let $X(t)=T(t)+S(t)+R(t)$, where under A3 the residual $R$ follows a stable linear Gaussian SEM. Then $R$ is geometrically $\beta$-mixing and satisfies A1, and under A2 and A3 the estimated residual satisfies $\gamma_{\widehat{R}}(k)=\gamma_R(k)+\mathcal{O}(\varepsilon)$ at every lag $k$.
\end{lemma}

\begin{proof}
Under A3, $R$ is a structural VAR whose contemporaneous part is acyclic, so it can be solved into a reduced-form stable VAR with Gaussian innovations. Stationary vector ARMA processes with absolutely continuous innovations are geometrically $\beta$-mixing~\citep{mokkadem1988mixing}, which gives the first claim. The raw $X(t)$ is not mixing, because $T$ has long memory and $S$ has non-decaying periodic correlations. With $\widehat{R}=X-\widehat{T}-\widehat{S}$, A2 lets $\widehat{T}$ and $\widehat{S}$ absorb this non-mixing energy, and A3 bounds the energy left in $\widehat{R}$ by $\varepsilon$, so $\gamma_{\widehat{R}}(k)=\gamma_R(k)+\mathcal{O}(\varepsilon)$. Under A3, CI tests use only covariances, so they behave on $\widehat{R}$ as on $R$ up to $\mathcal{O}(\varepsilon)$.
\end{proof}

\begin{theorem}[Identifiability under Linear Gaussianity]
\label{thm:identifiability}
Let the system be a linear Gaussian SEM satisfying A1--A4, and let $\Psi$ be a consistent constraint-based CD algorithm (e.g.\ PCMCI+~\citep{runge2020discovering}) whose CI tests have Type-I error $\alpha_n$ and Type-II error $\kappa_n$. Suppose every edge of $E^\star_R$ has partial correlation at least $c_{\min}>\varepsilon$ in magnitude. For $\widehat{G}=\bigcup_{K\in\{T,S,R\}}\Psi(\widehat{K})$,
\begin{equation}
    \mathbb{E}[\mathrm{SHD}(\widehat{G},G^\star)] \le C_1\,\varepsilon + C_2\,\alpha_n + C_3\,\kappa_n,
    \label{eq:shd}
\end{equation}
with $C_1,C_2,C_3$ depending on $d$, $\tau_{\max}$, and graph density.
\end{theorem}

\begin{proof}
The error has three sources. \emph{Leakage ($C_1\varepsilon$).} Under A3 zero covariance implies independence, and by Lemma~\ref{lem:projection} a true trend or seasonal influence projected onto the residual space has partial correlation $\mathcal{O}(\varepsilon)$, which a test with threshold $\tau>\varepsilon$ rejects in $G_R$. \emph{False positives ($C_2\alpha_n$).} By Lemma~\ref{lem:weak_dependence}, $R$ is $\beta$-mixing and $\widehat{R}$ matches its covariances up to $\mathcal{O}(\varepsilon)$, so each CI test on $\widehat{R}$ has Type-I error $\alpha_n$~\citep{rio2017asymptotic}. \emph{Missed edges ($C_3\kappa_n$).} An edge of $E^\star_R$ has partial correlation at least $c_{\min}>\tau$, so only a Type-II error can drop it. Each CI error changes a bounded number of adjacencies and orientations, so the constants depend on $d$, $\tau_{\max}$, and density. For a consistent test with $\varepsilon<\tau<c_{\min}$, $\alpha_n,\kappa_n\to0$, hence $\widehat{G}\to G^\star$ as $\varepsilon\to0$ and $n\to\infty$.
\end{proof}

\begin{corollary}[Trend-Seasonal Leakage]
\label{cor:trend}
Let $X_i=T_i+S_i+R_i$ satisfy A2 and A3 and suppose the true process contains a cross-scale pathway, e.g.\ $T_X(t)\to Y(t)$. After decomposition the induced effect appears mainly in $\widehat{T}_Y$, with at most $\mathcal{O}(\varepsilon)$ leakage into the seasonal and residual bands, so the edge enters the wrong component graph with probability $\mathcal{O}(\varepsilon)$.
\end{corollary}

\begin{proof}[Proof sketch]
Because $T_X$ concentrates near $\omega\approx0$ (A2), STL assigns its energy to $\widehat{T}_Y$, and its projection onto the seasonal and residual subspaces has variance $\mathcal{O}(\varepsilon)$ by A3, as in Lemma~\ref{lem:projection}. The sum of the two leakages is still $\mathcal{O}(\varepsilon)$.
\end{proof}

\noindent\textbf{Interpreting the bound.} Theorem~\ref{thm:identifiability} replaces the global stationarity that constraint-based methods implicitly need on raw data with one interpretable quantity, the leakage $\varepsilon$. Errors stay local, since by Corollary~\ref{cor:trend} a trend or seasonal misspecification moves into $G_R$ by at most $\mathcal{O}(\varepsilon)$. The bound is asymptotic, and a finite-$n$ version would follow from concentration inequalities for $\beta$-mixing sequences~\citep{merlevede2011bernstein,rio2017asymptotic}.

\section{The DCD Framework}
\label{sec:dcd_framework}

Guided by Theorem~\ref{thm:identifiability}, DCD keeps the leakage of A3 small through the decomposition and then runs component-specific analysis and graph integration (Algorithm~\ref{alg:dcd}, Figure~\ref{fig:DCD_framework}).

\begin{algorithm}[t]
\caption{Decomposition-based Causal Discovery (DCD)}
\label{alg:dcd}
\begin{algorithmic}[1]
\Require $\{X_i(t)\}_{i=1}^{d}$; period set $\mathcal{P}$; thresholds $\tau_S,\alpha$
\Ensure integrated graph $G_{\mathrm{combined}}$
\For{each $X_i(t)$} \Comment{Stage 1: decomposition}
    \State $P_i\gets\arg\max_{p\in\mathcal{P}}\mathrm{Var}(S_i^{(p)})$; STL: $X_i=T_i+S_i+R_i$
    \If{$\mathrm{Var}(S_i)<\tau_S$} set $S_i\gets 0$ (aperiodic) \EndIf
\EndFor
\For{each $i$} \Comment{Stage 2: component analysis}
    \If{ADF/KPSS declare $T_i$ non-stationary} add $t\to X_i$ to $G_T$ \EndIf
    \If{$S_i\neq0$ and HSIC$(S_i,t)$ gives $p<\alpha$} add $t\to X_i$ to $G_S$ \EndIf
\EndFor
\State $G_R\gets$ PCMCI+ on $\{R_i\}$ (lagged and contemporaneous)
\State $G_{\mathrm{combined}}\gets G_R\cup G_T\cup G_S$ \Comment{Stage 3: integration}
\State \Return $G_{\mathrm{combined}}$
\end{algorithmic}
\end{algorithm}

\subsection{Decomposition (A2, A3)}
We separate $T_i$, $S_i$, and $R_i$ with Seasonal-Trend decomposition based on LOESS (STL)~\citep{stl}, whose local smoothing adapts to non-linear trends and keeps the leakage $\varepsilon$ small, as A3 requires. The period of each variable is selected automatically. We fit STL for each integer period from $2$ to $50$ and keep the one whose seasonal component has the largest variance, provided that variance exceeds $\tau_S$,
\begin{equation}
    P_i^\star = \mathop{\arg\max}_{p\in\{2,\dots,50\}} \mathrm{Var}\!\bigl(S_i^{(p)}\bigr),
    \quad \text{s.t. } \mathrm{Var}(S_i^{(p)})\ge\tau_S ,
    \label{eq:period}
\end{equation}
with $\tau_S=0.1$. If no period clears the threshold, the variable is treated as aperiodic ($S_i=0$) so that STL does not fit a seasonal term to non-seasonal data. The final decomposition uses a seasonal smoother of length $15$ and a trend smoother of length $\max(P_i^\star+1,35)$, rounded up to an odd number, without robustness iterations.

\subsection{Component-Specific Causal Analysis}

\textbf{Trend.} $T_i(t)$ carries the low-frequency evolution driven by the time index. We test it with the Augmented Dickey-Fuller (ADF) test~\citep{dickey1979distribution} and the KPSS test~\citep{kpss}, both with a constant term at level $0.05$. The two tests have opposite nulls, a unit root for ADF and stationarity for KPSS. A variable is stationary when ADF rejects and KPSS does not, and non-stationary when the reverse holds. When the tests conflict, we treat the trend as nonlinear and therefore non-stationary. We add a time-proxy edge $t\to X_i$ to $G_T$ unless the verdict is stationary. This stage detects whether a variable carries a temporal trend, not trend-level causality between variables.

\textbf{Seasonal.} This stage records whether a variable carries periodic forcing that should be marked as $t\to X_i$. For a variable that passes the variance gate of Eq.~\eqref{eq:period}, we run a permutation HSIC test~\citep{gretton2005measuring} between $S_i(t)$ and $t$ (RBF kernel with $\gamma=0.1$ on both arguments, $500$ permutations) and add the edge when $p<0.05$. The test is confirmatory. Once STL has extracted a component at period $P$, $S_i$ is periodic by construction, so the test is almost always significant when $S_i$ has non-negligible variance. The decisions that matter are the period selection and the variance gate, which screen out pseudo-seasonal components fitted to noise. A discovery-oriented version would test $S_i$ against a measured seasonal driver or test coupling between $S_i$ and $S_j$.

\textbf{Residual.} After trend and seasonality are removed, $R_i(t)$ approximates a stationary, weakly dependent process (Lemma~\ref{lem:weak_dependence}). We apply PCMCI+~\citep{runge2020discovering}, whose iterative CI tests control for autocorrelation and target both lagged ($X_{t-\tau}\to Y_t$) and contemporaneous ($X_t\to Y_t$) edges. This is the only stage that recovers mechanistic edges, and it yields $G_R$ free of trend and seasonal confounding.

\subsection{Integration of Multi-Scale Graphs}
\label{sec:integration}
The final stage forms $G_{\mathrm{combined}}=(V\cup\{t\},\,E_T\cup E_S\cup E_R)$. The components occupy nearly disjoint bands (A2) and represent distinct mechanisms (A4), so the edge sets are additive and no conflict resolution is needed. The result is a multi-relational graph with time-proxy edges from $G_T$ and $G_S$ (exogenous drift and periodic forcing) and mechanistic edges from $G_R$. It inherits the guarantee of Theorem~\ref{thm:identifiability} and avoids the spurious transitive edges that arise when one algorithm is run on raw, mixed-scale data.

\subsection{Computational Complexity}
\label{sec:complexity}
For $d$ variables and $n$ time points, STL costs $O(d\,n\,L)$ for LOESS window $L$. The trend stage runs $O(d)$ ADF/KPSS tests of cost $O(n)$, and the seasonal stage runs $O(d)$ HSIC permutation tests of cost $O(n^2)$ each with a naive Gram matrix, reducible to near-linear with random features. The residual search dominates, at $O(d^2\,\tau_{\max}\,n)$ per CI test and growing with the conditioning sets as in PCMCI+~\citep{runge2020discovering}. The three component analyses are independent and can run in parallel. Removing trend and seasonality also shrinks the residual conditioning sets, since shared drift induces many weak long-range correlations in raw data.

\section{Experiments}
\label{sec:experiments}

We evaluate DCD on synthetic benchmarks with known ground truth and on two real-world datasets, against six baselines spanning constraint-based, score-based, and non-stationary CD, using TPR, FDR, and SHD (Section~\ref{sec:prelim}).

\subsection{Datasets}
\textbf{Synthetic.} We use the TimeGraph benchmark~\citep{TimeGraph}, built to include the trend- and seasonality-driven non-stationarity that most synthetic suites omit, with $n\in\{500,1000,1500\}$, $d\in\{4,6,8,10\}$, and lag order $l\in\{2,3,4\}$. The residual combines noise with instantaneous and lagged causal structure,
\begin{equation}
    R_i(t) = \sum_{j=1}^{i-1}\sum_{\tau=0}^{l}\alpha_{ij}^{\tau}\, R_j(t-\tau) + \xi_i(t),
    \label{eq:synth}
\end{equation}
with $\alpha_{ij}^{\tau}\in[0.3,0.8]$ and $\xi_i(t)\sim\mathcal{N}(0,0.1)$. Trend and seasonal terms give known long- and short-term structure, $X_i(t)=\mu_i+\beta_i t+A_i\sin(2\pi t/\omega_i)+R_i(t)$ with $\mu_i\in[5,12]$, $\beta_i\in[0,0.2]$, $A_i\in[0,4]$, and $\omega_i\in[12,25]$. Details are in Appendix~\ref{appendix:synthetic_data}.

\textbf{Arctic sea ice.} Monthly Arctic data from ERA5 reanalysis~\citep{hersbach2020era5} and NSIDC satellite observations~\citep{cavalieri1996sea}, 1979--2018, with atmospheric and oceanic variables area-weighted north of $25^\circ$N, strong seasonality, and long-term trends.

\textbf{ETTh1.} Hourly transformer load and oil temperature~\citep{zhou2021informer}, with strong seasonality, non-stationary trends, and non-linear dependence among the load and temperature variables.

The variables of both real-world datasets are listed in Appendix~\ref{appendix:datasets}.

\subsection{Baselines and Settings}
We compare against \textbf{PCMCI+}~\citep{runge2020discovering}, \textbf{CD-NOD}~\citep{huang2020causal}, \textbf{CFCI}~\citep{spirtes_causation, ramsey2006adjacency}, \textbf{CPC}~\citep{ramsey2006adjacency}, \textbf{BOSS-LiNGAM}~\citep{andrews2023fast, shimizu2011directlingam}, and \textbf{DYNOTEARS}~\citep{pamfil2020dynotears}. DCD's residual stage and PCMCI+ both use the CMI-knn test ($\mathrm{knn}{=}0.2$, $5$ shuffle neighbors, rank transform) at $\mathrm{pc\_alpha}=0.05$, and PCMCI+ uses the shuffle test and the majority collider rule. CD-NOD uses causal-learn's \texttt{cdnod} with the KCI test at $\alpha=0.05$, the stable option, $\mathrm{uc\_rule}=0$, $\mathrm{uc\_priority}=-1$, and the sample index as surrogate. CFCI, CPC, and BOSS-LiNGAM run in the Tetrad GUI with default settings (Fisher-Z, $\alpha=0.01$ for CFCI and CPC). DYNOTEARS uses \texttt{causalnex} \texttt{from\_pandas\_dynamic} with $\lambda_w=\lambda_a=0.05$ ($0.01$ on Arctic), weight threshold $0.05$, and lag order matched to the generative lag. No method was tuned to match DCD's sparsity, so the high baseline FDRs reflect default behavior under autocorrelation. CD-NOD returns no lagged edges, so it is evaluated on a non-temporal version of the synthetic data. The closest preprocessing competitor, differencing followed by PCMCI+, is covered by the residual-only isolation study (Section~\ref{sec:ablations}), and a cointegration-based baseline~\citep{engle1987cointegration,johansen1991estimation} is left to future work. Method descriptions and a settings table are in Appendix~\ref{appendix:baselines}.

\textbf{Protocol.} Each synthetic configuration is generated once per seed from a fixed seed set shared by all methods, so every method sees identical data, and we report means and standard deviations over seeds. Within DCD, the CI, HSIC, and stationarity tests all use $\alpha=0.05$. PCMCI+ applies its internal correction during condition selection~\citep{runge2020discovering}, and we add no graph-wide Benjamini-Hochberg correction (Section~\ref{sec:discussion}). The maximum lag $\tau_{\max}$ equals the generative lag on synthetic data and is $3$ for Arctic and $2$ for ETTh1. Every method receives the same $\tau_{\max}$, except DYNOTEARS on Arctic, which was run with lag order $2$, and ablation (ii) varies $\tau_{\max}$ to test robustness to an unknown lag. Edges are compared on the unrolled window graph, so a lagged and a contemporaneous edge between the same pair count as distinct.

\begin{figure}[t]
    \centering
    \includegraphics[width=0.85\textwidth]{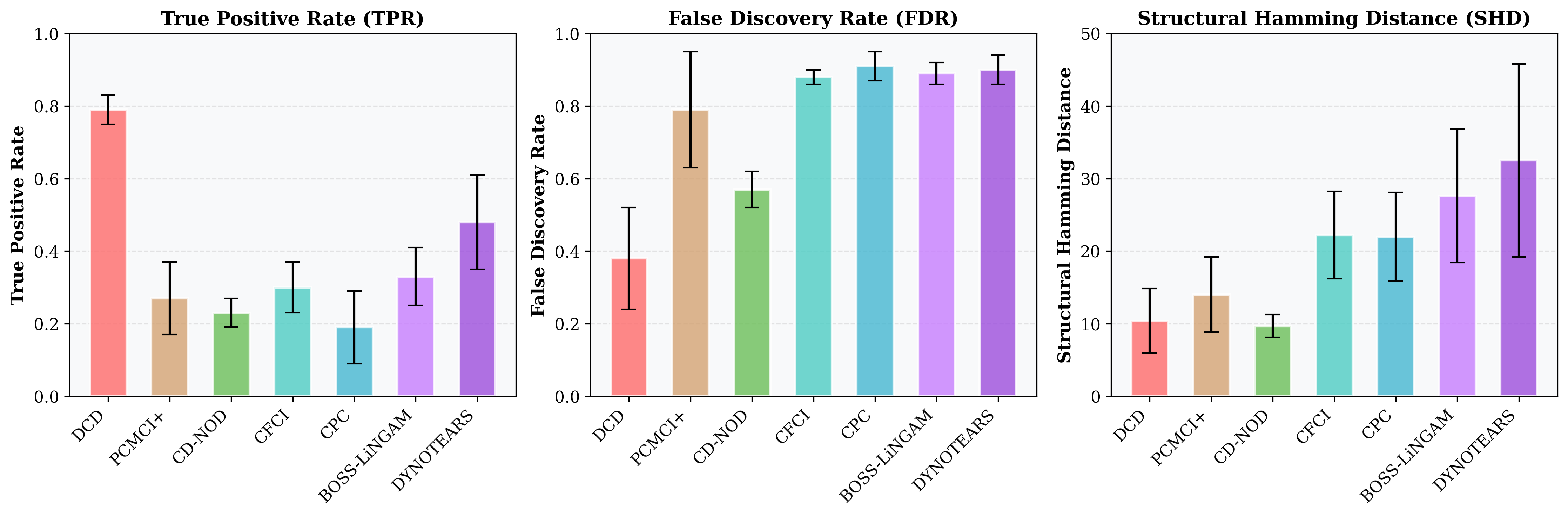}
    \caption{Aggregate performance on synthetic data. DCD gives the best balance of TPR, FDR, and SHD across conditions.}
    \label{fig:aggregated}
\end{figure}

\subsection{Results on Synthetic Data}
\label{sec:synthetic_results}

Figure~\ref{fig:aggregated} summarizes aggregate performance. DCD attains the best overall accuracy (mean TPR $0.79\pm0.04$, FDR $0.38\pm0.14$). CD-NOD attains the lowest SHD ($9.67\pm1.57$) only by discarding all lagged edges, which reduces its TPR to $0.23\pm0.04$. DYNOTEARS recovers about half the edges (TPR $0.48\pm0.13$), more than PCMCI+, CFCI, and CPC, but its FDR stays near $0.90$ and its SHD is the largest ($32.5\pm13.3$).

\textbf{Sample size.} Averaged over $d\in\{4,\dots,10\}$, DCD keeps TPR at $0.79$ as $n$ grows from $500$ to $1500$, while FDR rises only from $0.36$ to $0.40$ and SHD from $10.1$ to $10.8$ ($0.38$ and $10.3$ at $n=1000$). At $n=1500$, PCMCI+ reaches TPR $0.25$, FDR $0.82$ (up from $0.79$), and SHD $14.0$, and DYNOTEARS reaches $0.49$, $0.91$, and $35.2$ (SHD up from $28.4$). CD-NOD gives $0.23$, $0.57$, and $9.67$ at every $n$, and CFCI and CPC deteriorate more severely. The full table and plots are in Appendix~\ref{appendix:sample_size}.

\begin{table}[t]
\centering
\small
\renewcommand{\arraystretch}{1.05}
\setlength{\tabcolsep}{4pt}
\caption{Performance at different ground-truth lags (mean $\pm$ std over seeds). Best value per column in bold.}
\label{tab:stats_lag_value}
\resizebox{\linewidth}{!}{
\begin{tabular}{lccc ccc ccc}
\toprule
\textbf{Method}
& \multicolumn{3}{c}{\textbf{Lag 2}}
& \multicolumn{3}{c}{\textbf{Lag 3}}
& \multicolumn{3}{c}{\textbf{Lag 4}} \\
\cmidrule(lr){2-4} \cmidrule(lr){5-7} \cmidrule(lr){8-10}
& TPR $\uparrow$ & FDR $\downarrow$ & SHD $\downarrow$
& TPR $\uparrow$ & FDR $\downarrow$ & SHD $\downarrow$
& TPR $\uparrow$ & FDR $\downarrow$ & SHD $\downarrow$ \\
\midrule
\textbf{DCD}
& \textbf{0.78 $\pm$ 0.03} & \textbf{0.28 $\pm$ 0.13} & \textbf{7.83 $\pm$ 3.90}
& \textbf{0.79 $\pm$ 0.04} & \textbf{0.38 $\pm$ 0.15} & 10.42 $\pm$ 4.89
& \textbf{0.80 $\pm$ 0.05} & \textbf{0.48 $\pm$ 0.06} & 13.00 $\pm$ 3.64 \\
PCMCI+
& 0.21 $\pm$ 0.11 & 0.81 $\pm$ 0.11 & 10.83 $\pm$ 3.07
& 0.30 $\pm$ 0.11 & 0.72 $\pm$ 0.22 & 13.50 $\pm$ 6.26
& 0.29 $\pm$ 0.08 & 0.86 $\pm$ 0.04 & 17.75 $\pm$ 2.70 \\
CD-NOD
& 0.23 $\pm$ 0.04 & 0.57 $\pm$ 0.05 & 9.50 $\pm$ 1.73
& 0.23 $\pm$ 0.04 & 0.57 $\pm$ 0.05 & \textbf{9.75 $\pm$ 1.54}
& 0.23 $\pm$ 0.04 & 0.57 $\pm$ 0.05 & \textbf{9.75 $\pm$ 1.54} \\
CFCI
& 0.32 $\pm$ 0.06 & 0.88 $\pm$ 0.02 & 19.25 $\pm$ 6.18
& 0.30 $\pm$ 0.06 & 0.88 $\pm$ 0.01 & 24.00 $\pm$ 5.51
& 0.29 $\pm$ 0.08 & 0.88 $\pm$ 0.02 & 26.00 $\pm$ 5.70 \\
CPC
& 0.18 $\pm$ 0.11 & 0.91 $\pm$ 0.06 & 17.08 $\pm$ 4.78
& 0.19 $\pm$ 0.06 & 0.91 $\pm$ 0.03 & 23.08 $\pm$ 5.26
& 0.20 $\pm$ 0.11 & 0.92 $\pm$ 0.03 & 26.25 $\pm$ 5.56 \\
BOSS-LiNGAM
& 0.35 $\pm$ 0.08 & 0.89 $\pm$ 0.05 & 21.58 $\pm$ 5.21
& 0.32 $\pm$ 0.09 & 0.89 $\pm$ 0.02 & 28.42 $\pm$ 5.71
& 0.31 $\pm$ 0.08 & 0.90 $\pm$ 0.02 & 32.42 $\pm$ 13.50 \\
DYNOTEARS
& 0.56 $\pm$ 0.10 & 0.87 $\pm$ 0.02 & 25.33 $\pm$ 1.53
& 0.52 $\pm$ 0.09 & 0.91 $\pm$ 0.02 & 29.17 $\pm$ 4.23
& 0.44 $\pm$ 0.10 & 0.93 $\pm$ 0.00 & 38.67 $\pm$ 6.81 \\
\bottomrule
\end{tabular}}
\\[2pt]
{\footnotesize CD-NOD returns no lagged edges, so its low SHD reflects abstention on the lagged structure rather than accurate recovery.}
\end{table}

\begin{table}[t]
\centering
\small
\renewcommand{\arraystretch}{1.1}
\setlength{\tabcolsep}{4pt}
\caption{Performance at different numbers of variables $d$ (mean $\pm$ std over seeds). Best value per column in bold.}
\label{tab:stats_i_value}
\resizebox{\textwidth}{!}{
\begin{tabular}{lccc ccc ccc ccc}
\toprule
\textbf{Method}
& \multicolumn{3}{c}{\textbf{$d = 4$}}
& \multicolumn{3}{c}{\textbf{$d = 6$}}
& \multicolumn{3}{c}{\textbf{$d = 8$}}
& \multicolumn{3}{c}{\textbf{$d = 10$}} \\
\cmidrule(lr){2-4} \cmidrule(lr){5-7} \cmidrule(lr){8-10} \cmidrule(lr){11-13}
& TPR $\uparrow$ & FDR $\downarrow$ & SHD $\downarrow$
& TPR $\uparrow$ & FDR $\downarrow$ & SHD $\downarrow$
& TPR $\uparrow$ & FDR $\downarrow$ & SHD $\downarrow$
& TPR $\uparrow$ & FDR $\downarrow$ & SHD $\downarrow$ \\
\midrule
\textbf{DCD} & \textbf{0.84 $\pm$ 0.06} & \textbf{0.20 $\pm$ 0.15} & \textbf{4.11 $\pm$ 2.32}
             & \textbf{0.77 $\pm$ 0.00} & \textbf{0.42 $\pm$ 0.11} & \textbf{10.78 $\pm$ 3.49}
             & \textbf{0.79 $\pm$ 0.00} & \textbf{0.46 $\pm$ 0.06} & 12.44 $\pm$ 2.30
             & \textbf{0.76 $\pm$ 0.00} & \textbf{0.44 $\pm$ 0.04} & 14.33 $\pm$ 1.58 \\
PCMCI+       & 0.34 $\pm$ 0.16 & 0.69 $\pm$ 0.26 & 8.44 $\pm$ 4.50
             & 0.19 $\pm$ 0.10 & 0.90 $\pm$ 0.05 & 14.11 $\pm$ 3.72
             & 0.25 $\pm$ 0.00 & 0.80 $\pm$ 0.05 & 14.78 $\pm$ 2.64
             & 0.30 $\pm$ 0.00 & 0.79 $\pm$ 0.06 & 18.78 $\pm$ 3.35 \\
CD-NOD       & 0.23 $\pm$ 0.01 & 0.50 $\pm$ 0.00 & 8.67 $\pm$ 0.50
             & 0.18 $\pm$ 0.00 & 0.60 $\pm$ 0.00 & 12.00 $\pm$ 0.00
             & 0.29 $\pm$ 0.00 & 0.60 $\pm$ 0.00 & \textbf{8.00 $\pm$ 0.00}
             & 0.22 $\pm$ 0.00 & 0.60 $\pm$ 0.00 & \textbf{10.00 $\pm$ 0.00} \\
CFCI         & 0.26 $\pm$ 0.06 & 0.89 $\pm$ 0.01 & 18.56 $\pm$ 5.57
             & 0.35 $\pm$ 0.10 & 0.87 $\pm$ 0.02 & 18.67 $\pm$ 1.66
             & 0.29 $\pm$ 0.06 & 0.88 $\pm$ 0.02 & 24.56 $\pm$ 4.53
             & 0.31 $\pm$ 0.03 & 0.88 $\pm$ 0.02 & 30.56 $\pm$ 3.24 \\
CPC          & 0.07 $\pm$ 0.06 & 0.95 $\pm$ 0.03 & 18.44 $\pm$ 4.98
             & 0.23 $\pm$ 0.12 & 0.90 $\pm$ 0.04 & 17.11 $\pm$ 3.14
             & 0.24 $\pm$ 0.04 & 0.88 $\pm$ 0.04 & 24.67 $\pm$ 5.41
             & 0.20 $\pm$ 0.00 & 0.91 $\pm$ 0.02 & 28.33 $\pm$ 4.36 \\
BOSS-LiNGAM  & 0.26 $\pm$ 0.06 & 0.91 $\pm$ 0.02 & 21.22 $\pm$ 4.02
             & 0.32 $\pm$ 0.08 & 0.89 $\pm$ 0.02 & 20.00 $\pm$ 2.50
             & 0.31 $\pm$ 0.08 & 0.90 $\pm$ 0.03 & 31.67 $\pm$ 7.76
             & 0.41 $\pm$ 0.03 & 0.87 $\pm$ 0.04 & 37.00 $\pm$ 10.56 \\
DYNOTEARS    & 0.75 $\pm$ 0.00 & 0.81 $\pm$ 0.02 & 13.67 $\pm$ 1.53
             & 0.56 $\pm$ 0.10 & 0.87 $\pm$ 0.02 & 25.33 $\pm$ 1.53
             & 0.38 $\pm$ 0.00 & 0.90 $\pm$ 0.01 & 31.00 $\pm$ 3.46
             & 0.22 $\pm$ 0.02 & 0.91 $\pm$ 0.01 & 35.00 $\pm$ 4.27 \\
\bottomrule
\end{tabular}}
\end{table}

\textbf{Temporal lag.} As the generative lag grows (Table~\ref{tab:stats_lag_value}), DCD holds TPR in $[0.78,0.80]$ while FDR rises from $0.28$ to $0.48$. PCMCI+ degrades at longer lags, CFCI, CPC, and BOSS-LiNGAM incur rapidly growing SHD, and DYNOTEARS declines most (TPR $0.56\to0.44$, SHD $25.3\to38.7$).

\textbf{Dimensionality.} As $d$ grows from $4$ to $10$ (Table~\ref{tab:stats_i_value}), all methods lose accuracy, but DCD keeps the best TPR/FDR balance and continues to recover lagged edges. CD-NOD keeps a low SHD only by discarding lagged structure, and DYNOTEARS collapses (TPR $0.75\to0.22$, SHD more than doubling). The absolute FDR of DCD remains moderate ($\approx0.38$--$0.50$), which we return to in Section~\ref{sec:discussion}. Per-condition plots and a method-by-method reading of all three sweeps are in Appendix~\ref{appendix:detailed}.

\textbf{Why the baselines fail.} On raw data, PCMCI+, CFCI, and CPC inherit a shared low-frequency trend that reads as broad positive dependence, so their CI tests keep many edges and FDR is high. DYNOTEARS fits the same trend-driven covariance under its $\ell_1$ and acyclicity objective and adds edges as $n$ and $d$ grow. CD-NOD absorbs structure into its surrogate index and omits lagged edges. BOSS-LiNGAM needs non-Gaussianity for orientation, but the aggregate trend-plus-seasonal signal is near-Gaussian and carries no lags. DCD avoids all four failure modes by removing the low-frequency energy before the CI tests and encoding temporal forcing as explicit $t\to X$ edges.

\textbf{Runtime.} At $d{=}8$ and $n{=}1000$ on one CPU core, DCD takes about $6$\,s (STL $0.4$\,s, component tests $0.6$\,s, residual search $5.2$\,s). Because decomposition removes the drift that enlarges conditioning sets, this is below PCMCI+ on raw data ($\sim$$10$\,s) and CFCI/CPC ($\sim$$12$ to $14$\,s). DYNOTEARS is fastest ($\sim$$3$\,s) but least accurate. Runtime grows roughly linearly in $n$ and quadratically in $d$, matching the dominant $O(d^2\,\tau_{\max}\,n)$ term of Section~\ref{sec:complexity}. Timings depend on hardware.

\begin{figure}[t]
    \centering
    \includegraphics[width=0.85\textwidth]{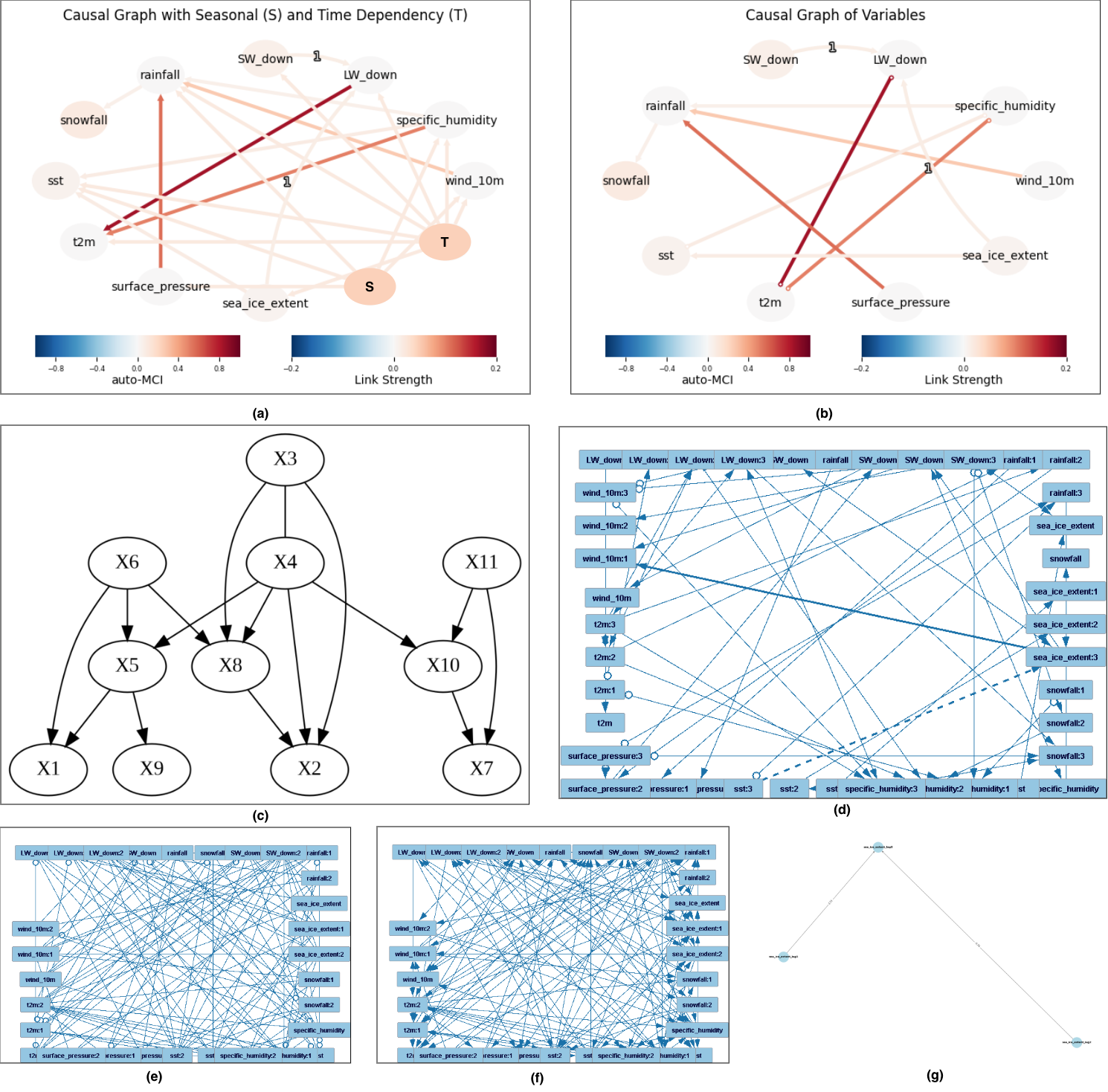}
    \caption{Causal graphs inferred by (a) DCD, (b) PCMCI+, (c) CD-NOD, (d) BOSS-LiNGAM, (e) CFCI, (f) CPC, and (g) DYNOTEARS on Arctic sea ice. DCD captures seasonal (S) and time-dependent (T) influences explicitly.}
    \label{fig:SIE}
\end{figure}

\subsection{Real-World Case Studies}
\label{sec:case_studies}

\textbf{Arctic sea ice.} DCD recovers multi-scale dependencies consistent with established climate dynamics. STL isolates dominant 12-, 36-, and 48-month periodicities, ADF/KPSS confirm stationarity of the extracted trends, and the residual search ($p<0.05$, lag $\le3$) recovers documented relationships, including the coupling of sea ice loss and surface warming~\citep{screen2010central}, temperature-humidity-radiation coupling~\citep{francis2006new}, and one-month-lagged responses of sea ice to SST and radiative forcing (Figure~\ref{fig:SIE}a). PCMCI+ on the raw data (Figure~\ref{fig:SIE}b) yields spurious or unoriented links and misses SST$\to$sea ice. CD-NOD (Figure~\ref{fig:SIE}c) recovers some contemporaneous links but mostly ties variables to the time index. DYNOTEARS (Figure~\ref{fig:SIE}g) collapses to a sea-ice autoregression and discards the atmospheric and oceanic drivers. BOSS-LiNGAM, CFCI, and CPC (Figure~\ref{fig:SIE}d--f) produce dense graphs typical of autocorrelation-driven false positives, whereas DCD yields a sparse, domain-consistent structure. Appendix~\ref{appendix:qualitative} discusses each graph in more detail.

\textbf{Electricity (ETTh1).} STL confirms strong periodicity (HSIC $p<0.001$) and the stationarity tests validate the components. The residual search (lag $\le2$, $\alpha=0.05$) yields sparse edges, namely autoregressive effects (LUFL$_{-1}\!\to$LUFL), cross-load interactions (HUFL$_{-1}\!\to$MUFL), and thermal persistence (OT$_{-1}\!\to$OT), with no load$\to$oil-temperature edge, consistent with independent short-term evolution of oil temperature. PCMCI+ produces conflicting orientations (e.g.\ LUFL$\leftrightarrow$HUFL), DYNOTEARS a dense graph of more than $30$ edges linking incompatible load channels, CFCI and CPC dense and noisy graphs, and BOSS-LiNGAM no lagged structure. Without ground truth, both studies are qualitative.

\subsection{Ablation Studies}
\label{sec:ablations}
We run four ablations to isolate the sources of DCD's gains. Full tables are in Appendix~\ref{appendix:ablations}.

\textbf{(i) Period misalignment.} Perturbing the STL period $P$ around the true periodicities, accuracy is best at alignment and degrades smoothly. At $d{=}4$, the optimal $P{=}30$ gives SHD $1.67$, and a misaligned $P{=}10$ still gives SHD $8.33$, below every baseline in Table~\ref{tab:stats_i_value} (Table~\ref{tab:ablation_period}).

\textbf{(ii) Causal depth and dimensionality.} Across $d\in\{4,6,8\}$ and $\tau_{\max}\in\{2,4,6\}$, SHD is nearly flat at $d{=}4,6$ and recall is maintained as $\tau_{\max}$ grows, so erring large on $\tau_{\max}$ is safe at moderate $d$ (Table~\ref{tab:lag_ablation}). Isolating residuals suppresses the FDR explosion that affects constraint-based search at long horizons.

\begin{table}[t]
\centering
\small
\setlength{\tabcolsep}{6pt}
\caption{Isolation test ($d{=}6$, $n{=}1000$, mean $\pm$ std over 3 seeds). ``Raw'' keeps trend and seasonality, ``Residuals'' runs the method on the STL residuals $R_i(t)$ only, and ``Multi-scale'' is full DCD.}
\label{tab:isolation_ablation}
\begin{tabular}{llccc}
\toprule
\textbf{Method} & \textbf{Input} & TPR $\uparrow$ & FDR $\downarrow$ & SHD $\downarrow$ \\
\midrule
PCMCI+    & Raw       & 0.45 $\pm$ 0.08 & 0.81 $\pm$ 0.02 & 28.5 $\pm$ 2.1 \\
PCMCI+    & Residuals & 0.83 $\pm$ 0.11 & 0.59 $\pm$ 0.04 & 11.2 $\pm$ 1.6 \\
DYNOTEARS & Raw       & 0.67 $\pm$ 0.05 & 0.85 $\pm$ 0.02 & 25.0 $\pm$ 1.8 \\
DYNOTEARS & Residuals & 0.79 $\pm$ 0.06 & 0.68 $\pm$ 0.03 & 14.8 $\pm$ 1.7 \\
\textbf{DCD} & \textbf{Multi-scale} & \textbf{1.00 $\pm$ 0.00} & \textbf{0.50 $\pm$ 0.09} & \textbf{6.0 $\pm$ 1.7} \\
\bottomrule
\end{tabular}
\end{table}

\textbf{(iii) Multi-scale integration versus preprocessing.} This ablation is also our closest comparison to a differencing or residualization $+$ PCMCI+ pipeline (Table~\ref{tab:isolation_ablation}). Running PCMCI+ and DYNOTEARS on STL residuals improves both (PCMCI+: TPR $0.45\to0.83$, FDR $0.81\to0.59$, SHD $28.5\to11.2$; DYNOTEARS: TPR $0.67\to0.79$, FDR $0.85\to0.68$, SHD $25.0\to14.8$), so cleaning the series already helps. Neither matches full DCD (TPR $1.00$, SHD $6.0$), so combining time-proxy and residual edges is a separate source of gains, not a preprocessing artifact.

\textbf{(iv) Spectral separability.} Inducing leakage by amplitude modulation, $S_i(t)=A_i[1+\lambda T_i(t)]\sin(2\pi t/P)$, raises the measured $\varepsilon$ from $0.02$ to $0.25$ as $\lambda$ goes from $0$ to $1$. TPR declines monotonically from $1.00$ to $0.72$, FDR rises from $0.44$ to $0.69$, and SHD grows from $5.0$ to $13.5$ (Table~\ref{tab:spectral_sensitivity}). Even at the strongest coupling DCD keeps TPR above $0.7$, consistent with the linear-in-$\varepsilon$ bound of Theorem~\ref{thm:identifiability}.

\section{Discussion and Limitations}
\label{sec:discussion}

\textbf{When DCD helps.} The largest gains arise when the data carry strong, separable trend or seasonal energy together with sparse short-term structure, as in climate, energy, and macro-finance, where constraint-based search on raw data collapses into a dense graph. When a series is stationary and aperiodic, decomposition does little ($S_i=0$, flat trend) and DCD reduces to the residual search, so it loses no accuracy in the easy case. Cryosphere problems such as Arctic sea ice forecasting~\citep{hossain2025correlation} and supraglacial lake evolution in Greenland~\citep{hossain2025supraglacial} are natural next targets.

\textbf{Precision.} DCD attains the best FDR among all methods, but its absolute FDR ($\approx0.38$--$0.50$) is moderate and may grow on larger graphs, where candidate edges grow quadratically. A Benjamini-Hochberg correction across the residual CI $p$-values, or a per-test threshold above the leakage floor ($\tau>\varepsilon$), would trade recall for precision. We report the uncorrected setting so that all methods are compared on equal terms.

\textbf{Time-proxy edges.} A $t\to X$ edge records that a variable is driven by some exogenous temporal process, such as a calendar cycle or secular trend, without committing to its physical origin. This stops two variables that share a seasonal driver from being joined directly and keeps $G_R$ readable as the short-term mechanism. A practitioner who can name the driver, for instance insolation, may replace $t$ by the measured variable and obtain a fully mechanistic graph.

\textbf{Checking the assumptions.} A2--A4 cannot be guaranteed but can be checked in part. Band overlap (A2) can be read from each periodogram and summarized by the leakage $\varepsilon$ of Eq.~\eqref{eq:leak}, and residual-to-trend feedback (A4) can be screened by testing the residual against the trend at positive lags. With the same STL settings, $\varepsilon\approx0.02$ on synthetic data, $0.08$ on Arctic, and $0.15$ on ETTh1, almost all of it in the trend band (largest seasonal-to-residual correlation $\approx0.05$). Aperiodic variables are excluded from this measure, since the fallback assigns the raw series to the trend and would inflate $\varepsilon$. All three values lie in the range where ablation (iv) keeps TPR high. In practice three choices matter, namely the candidate period set $\mathcal{P}$ (seeded with domain knowledge), $\tau_{\max}$ (erring large is safe at moderate $d$), and $\alpha$ (conservative tests, $\tau>\varepsilon$, when separability is weak).

\textbf{Limitations.}
\label{sec:limitations}
(i)~Theorem~\ref{thm:identifiability} rests on strong sufficient assumptions (A2--A4), so it does not cover strong cross-frequency coupling or residual-to-trend feedback. (ii)~Lemmas~\ref{lem:projection} and~\ref{lem:weak_dependence} control leakage only through covariances, and a finite-sample characterization of STL's leakage $\varepsilon$ remains open. (iii)~The trend and seasonal stages estimate only time-proxy edges, so cross-variable structure at the trend or seasonal scale (e.g.\ $T_X\to T_Y$) is not recovered. (iv)~No graph-wide FDR control is applied. (v)~Baseline coverage is partial. Time-varying methods such as CDANs~\citep{Cdans} and CD-NOTS~\citep{sadeghi2025cdnots}, cointegration models~\citep{engle1987cointegration,johansen1991estimation}, TiMINo~\citep{peters2013causal}, and spectral methods~\citep{shajarisales2015telling,geweke1982measurement} are discussed but not benchmarked. (vi)~DCD needs a stable residual mechanism and a record long enough for STL. When the residual mechanism itself changes across regimes, a regime-aware method is needed, or DCD can be applied within each segment after changepoint detection.

\section{Conclusion}
\label{sec:conclusion}
We introduced DCD, which runs component-specific causal analysis on the trend, seasonal, and residual parts of each series and merges the results into a multi-scale causal graph that separates exogenous temporal forcing ($t\to X$) from short-term mechanism ($X\to Y$). Casting decomposition as spectral filtering gives identifiability guarantees under separability and linear Gaussian dynamics, an SHD bound with a leakage term and Type-I and Type-II error terms, and a corollary that keeps cross-scale errors at $\mathcal{O}(\varepsilon)$. Across synthetic benchmarks and two real datasets, DCD recovers causal structure more accurately than PCMCI+, CD-NOD, DYNOTEARS, CFCI, CPC, and BOSS-LiNGAM over a range of sample sizes, lags, and dimensions, at competitive runtime. An isolation study shows that decomposition alone does not explain the gains. Future work will relax A2--A4, derive finite-sample leakage bounds for STL, add cointegration-based and time-varying baselines with FDR control, and learn decompositions that minimize $\varepsilon$ in strongly non-linear systems.

\bibliographystyle{plainnat}
\bibliography{main}

\appendix

\section{Datasets}
\label{appendix:datasets}

Table~\ref{tab:datasets_all} summarizes the synthetic design and lists the variables of the two real-world datasets. The synthetic sweep is a full grid over sample size, dimension, and lag order. The Arctic dataset has ten variables, and ETTh1 has six load channels and the oil temperature.

\begin{table}[ht]
\centering
\small
\caption{Synthetic design and real-world variables.}
\label{tab:datasets_all}
\renewcommand{\arraystretch}{1.05}
\setlength{\tabcolsep}{4pt}
\footnotesize
\begin{minipage}[t]{0.27\textwidth}
\centering
\textbf{Synthetic grid}\\[3pt]
\begin{tabular}{@{}ll@{}}
\toprule
Factor & Levels \\
\midrule
$n$ & 500, 1000, 1500 \\
$d$ & 4, 6, 8, 10 \\
Lag $l$ & 2, 3, 4 \\
Comp. & T, S, R \\
\bottomrule
\end{tabular}
\end{minipage}%
\hfill
\begin{minipage}[t]{0.40\textwidth}
\centering
\textbf{Arctic sea ice (monthly)}\\[3pt]
\begin{tabular}{@{}ll@{}}
\toprule
Variable & Description \\
\midrule
\texttt{surface\_pressure} & Surface pressure \\
\texttt{wind\_10m} & 10\,m wind speed \\
\texttt{specific\_humidity} & Specific humidity \\
\texttt{t2m} & 2\,m air temperature \\
\texttt{SW\_down} & Shortwave rad.\ (down) \\
\texttt{LW\_down} & Longwave rad.\ (down) \\
\texttt{rainfall} & Rain rate \\
\texttt{snowfall} & Snowfall rate \\
\texttt{sst} & Sea surface temperature \\
\texttt{sea\_ice\_extent} & Sea ice extent \\
\bottomrule
\end{tabular}
\end{minipage}%
\hfill
\begin{minipage}[t]{0.28\textwidth}
\centering
\textbf{ETTh1 (hourly)}\\[3pt]
\begin{tabular}{@{}ll@{}}
\toprule
Variable & Description \\
\midrule
HUFL & High useful load \\
HULL & High useless load \\
MUFL & Middle useful load \\
MULL & Middle useless load \\
LUFL & Low useful load \\
LULL & Low useless load \\
OT & Oil temperature \\
\bottomrule
\end{tabular}
\end{minipage}
\\[4pt]
{\footnotesize T = trend, S = seasonal, R = residual.}
\end{table}

\section{Baselines and Settings}
\label{appendix:baselines}

\textbf{PCMCI+}~\citep{runge2020discovering} is a constraint-based method for time series that recovers lagged and contemporaneous edges, using momentary CI tests to control for autocorrelation. It runs on the raw series, so trend and seasonal energy enter its tests.

\textbf{CD-NOD}~\citep{huang2020causal} is a constraint-based method for heterogeneous and non-stationary data that adds a surrogate time or domain index as an extra variable. It does not model lagged relationships, so we evaluate it on a non-temporal version of the synthetic data.

\textbf{CFCI}~\citep{spirtes_causation, ramsey2006adjacency} is the conservative variant of FCI. It allows latent confounders and leaves an orientation undecided when the separating sets disagree.

\textbf{CPC}~\citep{ramsey2006adjacency} is the conservative variant of PC. It orients a collider only when all relevant separating sets agree, and marks the remaining triples as ambiguous.

\textbf{BOSS-LiNGAM}~\citep{andrews2023fast, shimizu2011directlingam} combines the best order score search (BOSS), which estimates a Markov equivalence class, with LiNGAM-style orientation that relies on non-Gaussian noise.

\textbf{DYNOTEARS}~\citep{pamfil2020dynotears} is score-based. It jointly estimates a contemporaneous weight matrix and lagged weight matrices under $\ell_1$ penalties, with a continuous acyclicity constraint on the contemporaneous part.

Table~\ref{tab:algo_params} lists the settings used in all experiments.

\begin{table}[ht]
\centering
\caption{Method settings. On synthetic data every method uses the generative lag as $\tau_{\max}$. On real data $\tau_{\max}=3$ for Arctic and $2$ for ETTh1, except DYNOTEARS on Arctic (lag order $2$).}
\label{tab:algo_params}
\small
\renewcommand{\arraystretch}{1.15}
\setlength{\tabcolsep}{6pt}
\begin{tabular}{p{2.6cm} p{11.6cm}}
\toprule
\textbf{Method} & \textbf{Settings} \\
\midrule
\textbf{DCD}
& STL with period $P_i^\star\in\{2,\dots,50\}$ chosen by seasonal variance, variance gate $\tau_S=0.1$, seasonal smoother $15$, trend smoother $\max(P_i^\star+1,35)$ rounded up to odd, no robustness iterations. ADF and KPSS with a constant term at $\alpha=0.05$. Permutation HSIC with RBF kernels ($\gamma=0.1$), $500$ permutations, $\alpha=0.05$. Residual stage PCMCI+ with CMI-knn ($\mathrm{knn}=0.2$, $5$ shuffle neighbors, rank transform), $\mathrm{pc\_alpha}=0.05$. \\
PCMCI+
& CMI-knn ($\mathrm{knn}=0.2$, $5$ shuffle neighbors, rank transform), shuffle significance test, $\mathrm{pc\_alpha}=0.05$, majority collider rule, no graph-wide FDR correction. \\
CD-NOD
& causal-learn \texttt{cdnod}, KCI test, $\alpha=0.05$, stable option, $\mathrm{uc\_rule}=0$, $\mathrm{uc\_priority}=-1$, sample index as surrogate. \\
CFCI, CPC
& Tetrad GUI defaults, Fisher-Z test, $\alpha=0.01$. \\
BOSS-LiNGAM
& Tetrad GUI defaults. \\
DYNOTEARS
& \texttt{causalnex} \texttt{from\_pandas\_dynamic}, $\lambda_w=\lambda_a=0.05$ ($0.01$ on Arctic), weight threshold $0.05$. \\
\bottomrule
\end{tabular}
\end{table}

\section{Synthetic Data Generation}
\label{appendix:synthetic_data}

Each dataset has $n$ time points and $d$ variables, with $n\in\{500,1000,1500\}$, $d\in\{4,6,8,10\}$, and lag order $l\in\{2,3,4\}$. The residual of Eq.~\eqref{eq:synth} gives each variable parents among the earlier variables, with an instantaneous term ($\tau=0$) and lagged terms ($\tau=1,\dots,l$), coefficients $\alpha_{ij}^{\tau}\in[0.3,0.8]$, and Gaussian noise $\xi_i(t)\sim\mathcal{N}(0,0.1)$. Each observed series adds a level, a linear trend, and a sinusoidal seasonal term,
\begin{equation}
    X_i(t) = \mu_i + \beta_i t + A_i \sin(2\pi t/\omega_i) + R_i(t),
\end{equation}
with $\mu_i\in[5,12]$, $\beta_i\in[0,0.2]$, $A_i\in[0,4]$, and $\omega_i\in[12,25]$. Variables combine these components in different ways, so the benchmark mixes trending, seasonal, and stationary series with both contemporaneous and lagged causal structure.

\section{Extended Ablation Results}
\label{appendix:ablations}

\subsection{Period Misalignment}
\label{appendix:period_ablation}

We fix the STL period $P$ instead of selecting it, and vary it around the periods of the generating seasonal terms across $d\in\{4,6,8\}$ (Table~\ref{tab:ablation_period}). Results are averaged over $n\in\{500,1000,1500\}$ at $\tau_{\max}=4$. A period that matches the generator (Opt.) gives the lowest SHD at every $d$. A mismatched period (Inc.) raises FDR, since seasonal energy that STL fails to capture stays in the residual, but recall does not collapse. At $d=4$, $P=30$ gives SHD $1.67$, and the worst mismatch, $P=10$, gives SHD $8.33$, which is still below every baseline at $d=4$ in Table~\ref{tab:stats_i_value}.

\begin{table}[ht]
\centering
\small
\caption{Effect of the STL period $P$ (mean $\pm$ std over $n\in\{500,1000,1500\}$, $\tau_{\max}=4$). Opt.\ marks periods that match the generator, Inc.\ marks mismatched periods.}
\label{tab:ablation_period}
\begin{tabular}{ccccc}
\toprule
$d$ & Period $P$ & TPR $\uparrow$ & FDR $\downarrow$ & SHD $\downarrow$ \\
\midrule
 & 10 (Inc.) & 0.500 $\pm$ 0.00 & 0.757 $\pm$ 0.04 & 8.33 $\pm$ 1.15 \\
 & 15 (Opt.) & 0.750 $\pm$ 0.00 & 0.565 $\pm$ 0.06 & 5.00 $\pm$ 1.00 \\
4 & 20 (Inc.) & 0.500 $\pm$ 0.00 & 0.747 $\pm$ 0.03 & 8.00 $\pm$ 1.00 \\
 & 30 (Opt.) & 0.750 $\pm$ 0.00 & 0.133 $\pm$ 0.23 & 1.67 $\pm$ 1.15 \\
 & 35 (Inc.) & 0.583 $\pm$ 0.14 & 0.720 $\pm$ 0.06 & 7.67 $\pm$ 1.15 \\
\midrule
 & 10 (Inc.) & 0.722 $\pm$ 0.10 & 0.705 $\pm$ 0.03 & 12.00 $\pm$ 1.00 \\
 & 15 (Opt.) & 1.000 $\pm$ 0.00 & 0.578 $\pm$ 0.04 & 8.33 $\pm$ 1.53 \\
6 & 20 (Inc.) & 0.722 $\pm$ 0.10 & 0.665 $\pm$ 0.03 & 10.33 $\pm$ 1.15 \\
 & 30 (Opt.) & 1.000 $\pm$ 0.00 & 0.490 $\pm$ 0.09 & 6.00 $\pm$ 2.00 \\
 & 35 (Inc.) & 0.667 $\pm$ 0.17 & 0.736 $\pm$ 0.04 & 13.00 $\pm$ 1.00 \\
\midrule
 & 10 (Inc.) & 0.500 $\pm$ 0.00 & 0.608 $\pm$ 0.06 & 10.33 $\pm$ 1.53 \\
 & 15 (Inc.) & 0.458 $\pm$ 0.07 & 0.616 $\pm$ 0.06 & 10.33 $\pm$ 1.53 \\
8 & 20 (Opt.) & 0.542 $\pm$ 0.07 & 0.468 $\pm$ 0.07 & 7.67 $\pm$ 1.15 \\
 & 30 (Inc.) & 0.417 $\pm$ 0.07 & 0.611 $\pm$ 0.03 & 10.00 $\pm$ 1.00 \\
 & 35 (Inc.) & 0.458 $\pm$ 0.07 & 0.662 $\pm$ 0.08 & 11.67 $\pm$ 2.08 \\
\bottomrule
\end{tabular}
\end{table}

\subsection{Causal Depth and Dimensionality}
\label{appendix:lag_ablation}

We vary the search depth $\tau_{\max}\in\{2,4,6\}$ at $d\in\{4,6,8\}$ with $P=30$, averaging over $n\in\{500,1000,1500\}$ (Table~\ref{tab:lag_ablation}). At $d=4$ and $d=6$, TPR does not change with $\tau_{\max}$ and SHD moves by at most one edge, so a search window larger than the true lag costs little. At $d=8$, TPR is lower at every depth (about $0.42$ to $0.46$), and SHD rises from $9.33$ to $10.67$ as $\tau_{\max}$ grows from $2$ to $6$. The $d=8$ rows use $P=30$, which Table~\ref{tab:ablation_period} shows is a mismatched period at that dimension, so they also carry the cost of misalignment.

\begin{table}[ht]
\centering
\small
\caption{Effect of the search depth $\tau_{\max}$ (mean $\pm$ std over $n\in\{500,1000,1500\}$, $P=30$).}
\label{tab:lag_ablation}
\begin{tabular}{ccccc}
\toprule
$d$ & $\tau_{\max}$ & TPR $\uparrow$ & FDR $\downarrow$ & SHD $\downarrow$ \\
\midrule
 & 2 & 0.750 $\pm$ 0.00 & 0.217 $\pm$ 0.20 & 2.00 $\pm$ 1.00 \\
4 & 4 & 0.750 $\pm$ 0.00 & 0.133 $\pm$ 0.23 & 1.67 $\pm$ 1.15 \\
 & 6 & 0.750 $\pm$ 0.00 & 0.300 $\pm$ 0.09 & 2.33 $\pm$ 0.58 \\
\midrule
 & 2 & 1.000 $\pm$ 0.00 & 0.444 $\pm$ 0.10 & 5.00 $\pm$ 1.73 \\
6 & 4 & 1.000 $\pm$ 0.00 & 0.490 $\pm$ 0.09 & 6.00 $\pm$ 2.00 \\
 & 6 & 1.000 $\pm$ 0.00 & 0.475 $\pm$ 0.09 & 5.67 $\pm$ 2.08 \\
\midrule
 & 2 & 0.458 $\pm$ 0.07 & 0.576 $\pm$ 0.02 & 9.33 $\pm$ 0.58 \\
8 & 4 & 0.417 $\pm$ 0.07 & 0.611 $\pm$ 0.03 & 10.00 $\pm$ 1.00 \\
 & 6 & 0.458 $\pm$ 0.07 & 0.630 $\pm$ 0.06 & 10.67 $\pm$ 1.53 \\
\bottomrule
\end{tabular}
\end{table}

\subsection{Sensitivity to Spectral Separability}
\label{app:spectral_sensitivity}

We induce leakage by modulating the seasonal amplitude with the trend,
\begin{equation}
    S_i(t) = A_i \bigl[1 + \lambda T_i(t)\bigr] \sin(2\pi t/P),
\end{equation}
with coupling $\lambda\in\{0,0.2,0.5,1\}$. As $\lambda$ grows, the trend spreads seasonal energy into sidebands around $2\pi/P$, part of which falls in the residual band, so A2 holds only approximately. Real systems show this pattern when the amplitude of a cycle drifts with a long-term trend. The experiment uses $d=6$, $n=1000$, $P=30$, and $\tau_{\max}=2$, averaged over 3 seeds. For each $\lambda$ we measure the realized leakage as the largest correlation between an estimated seasonal component and an estimated residual, which is the part of Eq.~\eqref{eq:leak} that the modulation affects.

Table~\ref{tab:spectral_sensitivity} shows a monotone decline. Leakage grows from $0.02$ to $0.25$, TPR falls from $1.00$ to $0.72$, FDR rises from $0.44$ to $0.69$, and SHD grows from $5.0$ to $13.5$. SHD grows roughly linearly in $\varepsilon$ over this range, which matches the form of the bound in Theorem~\ref{thm:identifiability}, and even at $\lambda=1$ DCD keeps TPR above $0.7$.

\begin{table}[ht]
\centering
\small
\caption{DCD under increasing spectral leakage (mean $\pm$ std over 3 seeds, $d=6$, $n=1000$, $P=30$, $\tau_{\max}=2$).}
\label{tab:spectral_sensitivity}
\begin{tabular}{ccccc}
\toprule
Coupling $\lambda$ & Leakage $\varepsilon$ & TPR $\uparrow$ & FDR $\downarrow$ & SHD $\downarrow$ \\
\midrule
0.0 & 0.02 & 1.000 $\pm$ 0.00 & 0.444 $\pm$ 0.10 & 5.0 $\pm$ 1.7 \\
0.2 & 0.06 & 0.944 $\pm$ 0.05 & 0.512 $\pm$ 0.06 & 6.8 $\pm$ 1.5 \\
0.5 & 0.12 & 0.861 $\pm$ 0.07 & 0.589 $\pm$ 0.05 & 9.2 $\pm$ 1.9 \\
1.0 & 0.25 & 0.722 $\pm$ 0.09 & 0.694 $\pm$ 0.04 & 13.5 $\pm$ 2.4 \\
\bottomrule
\end{tabular}
\end{table}

The isolation test of ablation (iii) is reported in full in Table~\ref{tab:isolation_ablation} of the main text.

\section{Detailed Results by Condition}
\label{appendix:detailed}

\subsection{Sample Size}
\label{appendix:sample_size}

Table~\ref{tab:stats_sample_size} and Figure~\ref{fig:ss} report each method at $n\in\{500,1000,1500\}$, averaged over the other factors.

\textbf{DCD.} TPR stays at $0.79$ at all three sizes, FDR rises from $0.36$ to $0.40$, and SHD from $10.08$ to $10.83$. With $\alpha$ fixed at $0.05$, a larger sample makes more weak residual dependencies significant, which is the likely source of the small FDR increase.

\textbf{PCMCI+.} TPR stays between $0.25$ and $0.28$ and FDR ends higher than it starts ($0.79$, $0.77$, $0.82$). More data does not reduce the error, which points to a bias rather than a variance problem. On raw data, shared trend and seasonal variation enter the CI tests as dependence, and more samples make that dependence easier to detect.

\textbf{CFCI and CPC.} FDR stays between $0.88$ and $0.91$ at every size. CFCI recovers $0.28$ to $0.32$ of the true edges and CPC $0.18$ to $0.19$.

\textbf{BOSS-LiNGAM.} TPR is $0.31$ to $0.34$ and FDR $0.89$ to $0.90$ at every size, while SHD falls slightly from $28.75$ to $26.42$.

\textbf{DYNOTEARS.} TPR rises from $0.44$ at $n=500$ to $0.50$ at $n=1000$ and $0.49$ at $n=1500$, FDR stays at $0.90$ to $0.91$, and SHD grows from $28.44$ to $35.22$. Additional data mostly adds false edges.

\textbf{CD-NOD.} Its output is identical across sizes (TPR $0.23$, FDR $0.57$, SHD $9.67$) because it returns only contemporaneous edges. Its SHD is the lowest in the table, but it recovers fewer than a quarter of the true edges, whereas DCD, at a slightly higher SHD, recovers more than three quarters.

\begin{table}[ht]
\centering
\small
\renewcommand{\arraystretch}{1.1}
\setlength{\tabcolsep}{4pt}
\caption{Performance at different sample sizes (mean $\pm$ std). Best value per column in bold.}
\label{tab:stats_sample_size}
\resizebox{\textwidth}{!}{
\begin{tabular}{lccc ccc ccc}
\toprule
\textbf{Method}
& \multicolumn{3}{c}{\textbf{$n = 500$}}
& \multicolumn{3}{c}{\textbf{$n = 1000$}}
& \multicolumn{3}{c}{\textbf{$n = 1500$}} \\
\cmidrule(lr){2-4} \cmidrule(lr){5-7} \cmidrule(lr){8-10}
& TPR $\uparrow$ & FDR $\downarrow$ & SHD $\downarrow$
& TPR $\uparrow$ & FDR $\downarrow$ & SHD $\downarrow$
& TPR $\uparrow$ & FDR $\downarrow$ & SHD $\downarrow$ \\
\midrule
\textbf{DCD} & \textbf{0.79 $\pm$ 0.03} & \textbf{0.36 $\pm$ 0.16} & 10.08 $\pm$ 4.74
             & \textbf{0.79 $\pm$ 0.05} & \textbf{0.38 $\pm$ 0.14} & 10.33 $\pm$ 4.72
             & \textbf{0.79 $\pm$ 0.04} & \textbf{0.40 $\pm$ 0.13} & 10.83 $\pm$ 4.69 \\
PCMCI+       & 0.27 $\pm$ 0.08 & 0.79 $\pm$ 0.16 & 14.17 $\pm$ 5.04
             & 0.28 $\pm$ 0.12 & 0.77 $\pm$ 0.18 & 13.92 $\pm$ 5.47
             & 0.25 $\pm$ 0.13 & 0.82 $\pm$ 0.13 & 14.00 $\pm$ 5.20 \\
CD-NOD       & 0.23 $\pm$ 0.04 & 0.57 $\pm$ 0.05 & \textbf{9.67 $\pm$ 1.61}
             & 0.23 $\pm$ 0.04 & 0.57 $\pm$ 0.05 & \textbf{9.67 $\pm$ 1.61}
             & 0.23 $\pm$ 0.04 & 0.57 $\pm$ 0.05 & \textbf{9.67 $\pm$ 1.61} \\
CFCI         & 0.32 $\pm$ 0.07 & 0.88 $\pm$ 0.02 & 24.58 $\pm$ 7.27
             & 0.28 $\pm$ 0.06 & 0.88 $\pm$ 0.01 & 22.58 $\pm$ 5.73
             & 0.30 $\pm$ 0.08 & 0.88 $\pm$ 0.01 & 22.08 $\pm$ 6.16 \\
CPC          & 0.19 $\pm$ 0.12 & 0.91 $\pm$ 0.04 & 22.75 $\pm$ 6.96
             & 0.19 $\pm$ 0.09 & 0.91 $\pm$ 0.05 & 22.25 $\pm$ 6.15
             & 0.18 $\pm$ 0.08 & 0.91 $\pm$ 0.04 & 21.42 $\pm$ 6.44 \\
BOSS-LiNGAM  & 0.34 $\pm$ 0.07 & 0.89 $\pm$ 0.03 & 28.75 $\pm$ 11.39
             & 0.31 $\pm$ 0.07 & 0.90 $\pm$ 0.03 & 27.25 $\pm$ 9.27
             & 0.32 $\pm$ 0.10 & 0.89 $\pm$ 0.03 & 26.42 $\pm$ 9.42 \\
DYNOTEARS    & 0.44 $\pm$ 0.16 & 0.90 $\pm$ 0.05 & 28.44 $\pm$ 9.68
             & 0.50 $\pm$ 0.13 & 0.90 $\pm$ 0.04 & 33.89 $\pm$ 14.32
             & 0.49 $\pm$ 0.12 & 0.91 $\pm$ 0.04 & 35.22 $\pm$ 15.79 \\
\bottomrule
\end{tabular}}
\end{table}

\begin{figure}[ht]
    \centering
    \includegraphics[width=\textwidth]{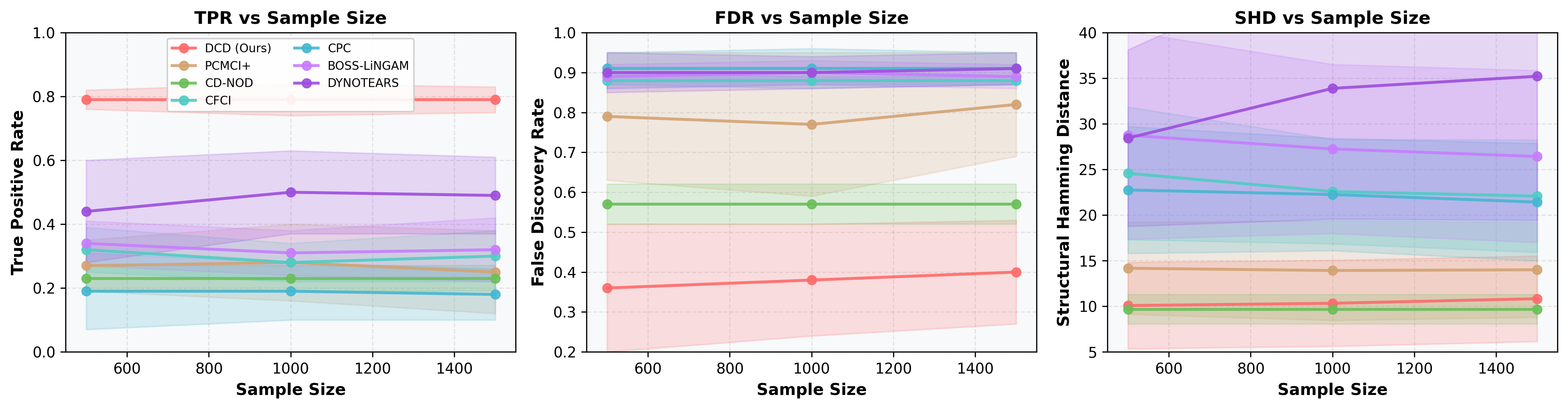}
    \caption{Effect of sample size on TPR, FDR, and SHD. Shaded bands show $\pm$ one standard deviation.}
    \label{fig:ss}
\end{figure}

\subsection{Temporal Lag}
\label{appendix:lag_effects}

Table~\ref{tab:stats_lag_value} and Figure~\ref{fig:lag} report each method as the generative lag grows from $2$ to $4$, with $\tau_{\max}$ set to the generative lag.

\textbf{DCD.} TPR stays between $0.78$ and $0.80$. FDR rises in near-equal steps from $0.28$ to $0.38$ to $0.48$, and SHD from $7.83$ to $13.00$. A longer window adds candidate parents for every variable, so more weak dependencies pass the tests, but FDR stays below $0.5$ throughout.

\textbf{PCMCI+.} TPR goes from $0.21$ to $0.30$ to $0.29$ and FDR from $0.81$ to $0.72$ to $0.86$, so neither moves monotonically, while SHD rises from $10.83$ to $17.75$. Larger conditioning sets make the CMI-knn tests less reliable at a fixed sample size.

\textbf{CFCI and CPC.} SHD rises steadily (CFCI $19.25\to26.00$, CPC $17.08\to26.25$), while TPR stays at or below $0.32$ and FDR at $0.88$ to $0.92$.

\textbf{BOSS-LiNGAM.} TPR declines from $0.35$ to $0.31$ and SHD rises from $21.58$ to $32.42$, with FDR at $0.89$ to $0.90$. Its orientation step uses non-Gaussianity rather than time order, so additional lags add candidate edges that it cannot screen.

\textbf{DYNOTEARS.} TPR falls from $0.56$ to $0.44$, FDR rises from $0.87$ to $0.93$, and SHD from $25.33$ to $38.67$, the largest increase of any method. DYNOTEARS estimates a separate $d\times d$ weight matrix for each lag, so its number of parameters grows with the window, and at a fixed sample size and penalty more weak edges survive the weight threshold.

\textbf{CD-NOD.} Its output barely changes with the lag (TPR $0.23$, FDR $0.57$, SHD $9.50$ to $9.75$) because it returns no lagged edges, which makes it a useful control for the lag effect on the other methods.

\begin{figure}[ht]
    \centering
    \includegraphics[width=\textwidth]{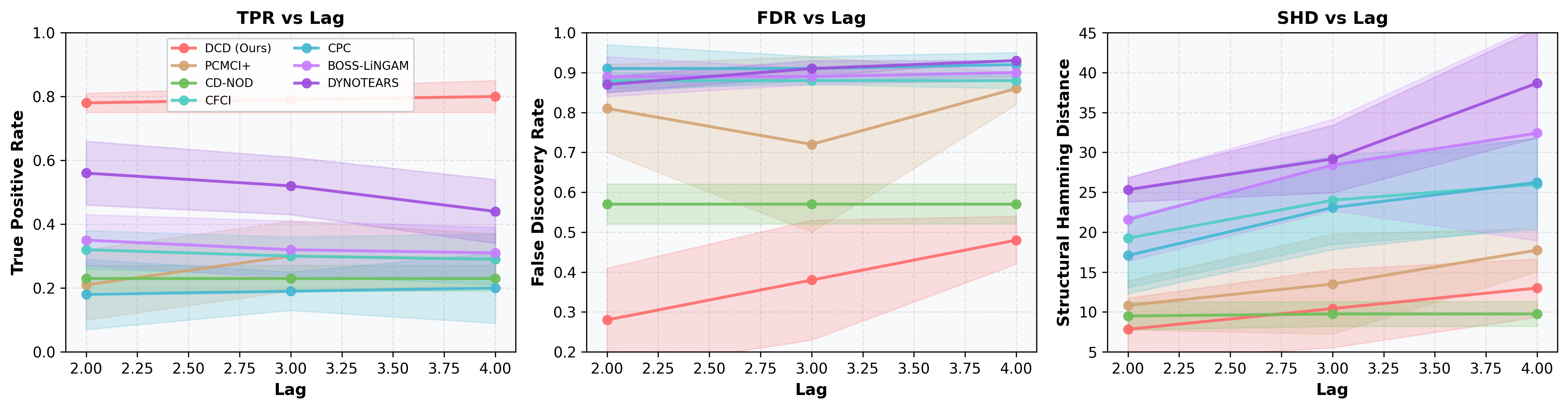}
    \caption{Effect of the generative lag on TPR, FDR, and SHD. Shaded bands show $\pm$ one standard deviation.}
    \label{fig:lag}
\end{figure}

\subsection{Dimensionality}
\label{appendix:dimensionality_effects}

Table~\ref{tab:stats_i_value} and Figure~\ref{fig:vars} report each method as $d$ grows from $4$ to $10$.

\textbf{DCD.} TPR stays between $0.76$ and $0.84$. FDR rises from $0.20$ at $d=4$ to $0.42$ to $0.46$ at $d\ge6$, and SHD from $4.11$ to $14.33$. At $d=10$ its FDR of $0.44$ is well below that of PCMCI+ ($0.79$), CFCI ($0.88$), CPC ($0.91$), BOSS-LiNGAM ($0.87$), and DYNOTEARS ($0.91$).

\textbf{PCMCI+.} TPR drops from $0.34$ at $d=4$ to $0.19$ at $d=6$ and recovers to $0.30$ at $d=10$, with FDR between $0.69$ and $0.90$ and SHD rising from $8.44$ to $18.78$.

\textbf{CFCI and CPC.} SHD grows from about $18$ to $30.56$ for CFCI and to $28.33$ for CPC, with TPR between $0.07$ and $0.35$ and FDR between $0.87$ and $0.95$.

\textbf{BOSS-LiNGAM.} TPR rises from $0.26$ to $0.41$, but FDR stays at $0.87$ to $0.91$ and SHD grows from $21.22$ to $37.00$, so the extra true edges come with many more false ones.

\textbf{DYNOTEARS.} TPR falls from $0.75$ to $0.22$, FDR rises from $0.81$ to $0.91$, and SHD grows from $13.67$ to $35.00$. At $d=10$ its TPR equals that of CD-NOD while its SHD is three and a half times larger. Its parameter count grows with $d^2$ times the number of lags, which at $d=10$ and a lag of $4$ is $500$ weights for $n=1000$ samples.

\textbf{CD-NOD.} TPR varies without a trend ($0.23$, $0.18$, $0.29$, $0.22$), FDR is $0.50$ to $0.60$, and SHD stays between $8.00$ and $12.00$, since it returns only contemporaneous edges.

\begin{figure}[ht]
    \centering
    \includegraphics[width=\textwidth]{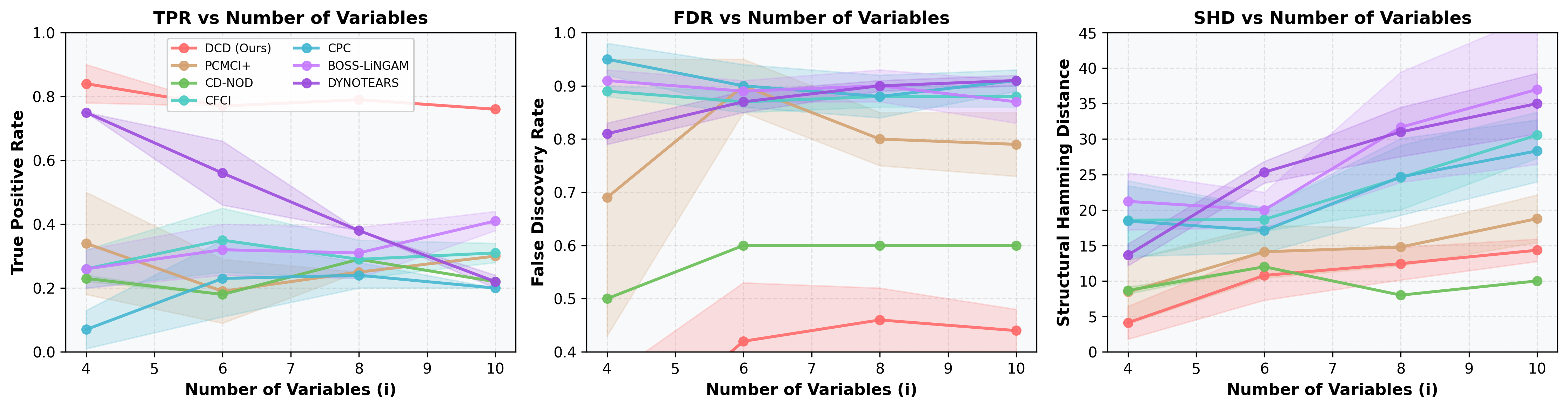}
    \caption{Effect of the number of variables on TPR, FDR, and SHD. Shaded bands show $\pm$ one standard deviation. The FDR panel starts at $0.4$, so DCD's value at $d=4$ ($0.20$, Table~\ref{tab:stats_i_value}) lies below the plotted range.}
    \label{fig:vars}
\end{figure}

\section{Qualitative Analysis of the Arctic Graphs}
\label{appendix:qualitative}

This section expands on the comparison in Figure~\ref{fig:SIE}. No ground truth exists for this dataset, so the comparison is qualitative.

\textbf{DCD (Figure~\ref{fig:SIE}a).} The graph is sparse. The time-proxy nodes T and S collect the long-term and periodic forcing, which leaves a small set of residual edges among the physical variables, including the one-month-lagged responses of sea ice to SST and radiative forcing. Because shared seasonal variation is routed through S rather than drawn between variables, the graph avoids the dense connectivity that a common annual cycle produces on raw data.

\textbf{PCMCI+ (Figure~\ref{fig:SIE}b).} Several edges are left unoriented, for instance between temperature and humidity, and the lagged SST$\to$sea ice edge is missing. On raw data the shared annual cycle appears in the CI tests as a contemporaneous association, which masks the weaker residual interactions.

\textbf{CD-NOD (Figure~\ref{fig:SIE}c).} CD-NOD flags several variables as non-stationary and links them to its surrogate time index. It recovers some contemporaneous links but no lagged structure, since it does not model lags.

\textbf{BOSS-LiNGAM, CFCI, and CPC (Figure~\ref{fig:SIE}d--f).} All three return dense graphs with edges between most pairs of variables and lags. Such density is typical of false positives driven by autocorrelation and shared trends. CFCI and CPC also leave many edges unoriented under their conservative orientation rules.

\textbf{DYNOTEARS (Figure~\ref{fig:SIE}g).} The output contains only sea ice extent at lags $0$, $1$, and $2$, with a lag-1 self-edge (weight $-1.75$) and a lag-2 self-edge (weight $-0.14$). All atmospheric and oceanic variables are dropped. Under strong trend and seasonality, the persistence of sea ice dominates the fit, and the penalized objective keeps only that signal.

\end{document}